\documentclass{article}

\usepackage{natbib}
\setcitestyle{square,comma,numbers,compress}
     \usepackage[preprint]{neurips_2020}

\usepackage[utf8]{inputenc} 
\usepackage[T1]{fontenc}    
\usepackage{hyperref}       
\usepackage{url}            
\usepackage{booktabs}       
\usepackage{amsfonts}       
\usepackage{nicefrac}       
\usepackage{microtype}      

\usepackage{algorithm,algorithmicx,algpseudocode} 
\usepackage{mathrsfs,amsmath,amssymb,amsthm}
\usepackage{multirow}
\usepackage{dsfont}
\usepackage{thmtools,thm-restate}
\usepackage{graphicx, caption, subcaption, wrapfig}

\DeclareMathOperator*{\argmax}{arg\,max}

\usepackage[toc]{appendix}

\newtheorem{theorem}{Theorem}
\newtheorem{definition}{Definition}

\newtheorem*{problem*}{Problem}
\newtheorem{fact}[theorem]{Fact}
\usepackage{thmtools,thm-restate}

\title{Multi-Armed Bandits with Local Differential Privacy}

\author{%
	Wenbo Ren\\
	Dept. Computer Science \& Engineering\\
	The Ohio State University\\
	\texttt{ren.453@osu.edu} 
	\And
	Xingyu Zhou\\
	Dept. Electrical \& Computer Engineering\\
	The Ohio State University\\
	\texttt{zhou.2055@osu.edu}
	\And
	Jia Liu \\
	Dept. Computer Science\\
	Iowa State University\\
	\texttt{jialiu@iastate.edu} 
	\And
	Ness B. Shroff \\
	Dept. ECE and CSE\\
	The Ohio State University\\
	\texttt{shroff.11@osu.edu} 
}

\begin{document}

\maketitle

\begin{abstract}
	This paper investigates the problem of regret minimization for multi-armed bandit (MAB) problems with local differential privacy (LDP) guarantee. In stochastic bandit systems, the rewards may refer to the users' activities, which may involve private information and the users may not want the agent to know. However, in many cases, the agent needs to know these activities to provide better services such as recommendations and news feeds. To handle this dilemma, we adopt differential privacy and study the regret upper and lower bounds for MAB algorithms with a given LDP guarantee. In this paper, we prove a lower bound and propose algorithms whose regret upper bounds match the lower bound up to constant factors. Numerical experiments also confirm our conclusions.
\end{abstract}

\section{Introduction}
\subsection{Background and motivation}
	\makeatletter
	\def\blfootnote{\gdef\@thefnmark{}\@footnotetext}
	\makeatother
	The 
	multi-armed bandit (MAB) \cite{EarlyBandit1985} problem provides a classic model for abstracting sequential decision making under uncertainty, and has attracted a wide range of interest in various areas, such as communication networks, online advertising, clinical trials, product testing, etc. In an MAB model, there is a set of \textit{arms}, and each\textit{ pull} of an arm generates a random \textit{reward} according to some \textit{unknown latent distribution} of this arm. The \textit{agent} adaptively chooses arms to pull according to past observations in order to achieve some goal. A widely studied goal is \textit{regret minimization}, where the regret is the expected gap between a proposed algorithm and an optimal algorithm that knows the latent distributions. To minimize the regret, the agent needs to balance the trade-off between \textit{exploration} and \textit{exploitation}, where exploration refers to learning the environment and exploitation refers to pulling the best arm according to the current knowledge. 
	
	In recent years, users have become increasingly concerned about protecting their private online information and activities, which may include their personal profiles, browsing histories, and activities on the Internet. They may not want to share this information with other parties. However, many real-world systems like medical experiments, recommender systems, advertisement allocators, online shopping websites, and search engines need such data to learn critical matters and provide better services. To handle this dilemma, there is a compelling need to develop algorithms that can optimally trade off system performance and the privacy level provided to the users. 
	
	A widely accepted and applied metric to measure the privacy level is the \textit{differential privacy} (DP) \cite{DPFoundation2014}, which, in theory, guarantees that it is difficult for any party or eavesdropper to determine whether or not an individual is listed in a private database. DP algorithms have been studied in many areas, such as data release \cite{DPDataRelease2011}, optimization \cite{DPOptimization2015}, and Q-learning \cite{DPQLearning2019}, just to name a few. However, DP remains under-explored in the MAB settings.
	
	Here, we take clinical trials as a concrete example to illustrate the use of DP in MAB. In an experiment of an illness with multiple treatments (aka arms), the experimenter (aka agent) wants to sequentially choose treatments for patients (aka individual users) based on past observations on treatment effects. This problem can be viewed as an MAB regret minimization problem. However, the patients may not be willing to share the actual effects of the treatments with the experimenter due to privacy concerns. By the DP bandit algorithms, the actual effects will not be known by the experimenter, which provides a certain level of privacy guarantee to all patients, while also enabling the experimenter to learn from the observations efficiently.
	
	The above example fits the \textit{local differentially private} (LDP) bandit model \cite{DPBanditLowerBound2019,DPBanditsCorrupt2018}. Different from the DP bandit model, in the LDP setting there is no trusted centroid curator \cite{DPFoundation2014}. In this paper, we assume that each user has its own curator (or privacy mechanism) that can do randomized mapping on its data to provide privacy guarantee. This curator can be softwares or plugins embedded in the user's devices or terminals, and the non-private data will not leave the control of the user unless they are processed by the user's curator. 
	
	Another example of LDP MAB is shopping websites, which also indicates the necessity of the LDP setting instead of the DP setting: The server wants to sequentially choose products to recommend according to the users' past purchase histories, while some users are not willing to share this information as the purchase histories may reveal private information (e.g., a person who buys a lot of heart-disease medicines is more likely to have related illness). In this scenario, it is unlikely that there is a third-party centroid curator that can gain access to all the purchase data, sine these data are commonly viewed as a valuable property for the company's business success. In the literature, the DP bandit problems have been studied in different settings \cite{DPBanditLowerBound2019,DPBanditsCorrupt2018,DPBanditMultiPartyContextual2019,DPBanditEmpirical2019,DPBanditWithContexual2014,DPBandit2015,DPBandit2019,DPBanditContextualSequential2018,DPBanditSequentialMultiAgent2015,DPBanditSequential2016,DPBanditAdcersarial2017}, while the LDP bandit problem remains under-explored.

\subsection{Problem formulation}
	\paragraph{Bandit model.} In this paper, our bandit model has $n$ arms indexed by $1,2,3,...,n$, and we use $[n]$\footnote{For any positive integer $m$, we define $[m]:=\{1,2,3,...,m\}$.} to denote the set of all arms. Each arm $a$ is associated with an \textit{unknown} latent distribution, and each pull of arm $a$ returns a random reward according to its latent distribution.
	We use $R_a^t$ to denote the reward of the $t$-th pull of arm $a$. For any arm $a$, the rewards $R_a^1,R_a^2,R_a^3......$ follow the same distribution, and we define $\mu_a := \mathbb{E}[R_a^1]$ as the \textit{mean reward} of arm $a$. We also assume that the rewards are \textit{independent} across arms and time, i.e., $(R_a^t, a\in[n], t\in \mathbb{Z}^+)$ are independent. Define $\mu^* = \max_{a\in[n]}\mu_a$. For any arm $a$, we define gap $\Delta_a := \mu^* - \mu_a$. An arm $a$ is said to be \textit{optimal} if $\Delta_a = 0$, and \textit{suboptimal} if $\Delta_a > 0$. 
	
	 	
	\paragraph{Regret minimization.} 
	Given a time horizon $T > n$ ($T$ may or may not be known, and in this paper, we assume not known), the agent pulls the arms for at most $T$ times. Let $A^t$ denote the $t$-th pulled arm. For any arm $a$ and time $t$, we use $N^t_a$ to denote the number of pulls on arm $a$ till time $t$, i.e., $N^t_a := \sum_{\tau=1}^t\mathds{1}\{A^\tau = a\}$. After $T$ pulls, the (expected) reward is $\sum_{t=1}^T\mathbb{E}[\mu_{A^t}] = \mathbb{E}\Big[\sum_{a\in[n]}N^T_a \mu_a\Big]$. 
	The notion of \emph{regret} is often used to measure the optimality gap between the rewards gained by the developed algorithm and an optimal algorithm that has prior knowledge of the best arm. Mathematically, the (pseudo) regret is defined as 
	\begin{align}
		R(T) := T\mu^* -  \sum_{t=1}^T\mathbb{E}[\mu_{A^t}] = \mathbb{E}[\sum_{a\in[n]}N^T_a \Delta_a]. \nonumber
	\end{align}
	The goal of the agent is to minimize the regret.

	\paragraph{Local differential privacy.} Before we define local differential privacy, we provide some preliminary definitions. We first introduce the notion of $\epsilon$-differential privacy \cite{DPFoundation2014}. We define the \textit{neighboring} data records as any two records that differ by only one entry. 
	\begin{definition}[$\epsilon$-differential privacy ($\epsilon$-DP)]\label{Def:DP}
		For $\epsilon > 0$, a randomized mapping $M : \mathcal{D} \rightarrow \mathbb{R}^l$ is said to be $\epsilon$-DP on $\mathcal{D} \subset \mathbb{R}^k$ if for any neighboring $x, x'$ in $\mathcal{D}$ and a measurable subset $E$ of $\mathbb{R}^l$, we have
		\begin{align}
			\mathbb{P}\{M(x) \in E\} \leq e^\epsilon \mathbb{P}\{M(x') \in E\}. \nonumber
		\end{align}
		The above inequality must also hold if we switch $x$ and $x'$. 
	\end{definition}
	This definition implies that for any neighboring records, after an $\epsilon$-DP mechanism, their statistical behaviors are similar. Hence, it is difficult for any party to determine which record is the source of the given output. Smaller values of $\epsilon$ implies higher levels of privacy. When $\epsilon = \infty$, there is no privacy.
	
	This paper focuses on the LDP bandit model, which can be described as follows: We split the parties into three categories: the agent, the curators, and the users. The users do not trust the agent. The curators stand between the users and the agent, providing privacy to the users and also help the agent to minimize the regret. In each iteration, the agent makes a decision on which arm to pull according to the knowledge of past private responses and sends a request to a user's curator. The curator then ``pulls the arm'' (e.g., awaiting the activity of the user), receives the reward, and returns a private response to the agent. In the LDP bandit model, the curators do not aggregate the rewards of the arms. To this end, for random vectors $X$ and $Y$, we use $X\in \sigma(Y)$ to represent that $X$ is determined by $Y$ plus some random factors independent of $Y$ and the bandit instance. In the following definition, it implies that the agent does not know the actual rewards. The formal definition of the LDP bandit model is stated in Definition~\ref{Def:BM-LDP}, where $\mathcal{D}$ is the domain of the rewards. 
	\begin{definition}[LDP bandit model]\label{Def:BM-LDP}
		Let $A^t$ be the $t$-th pulled arm, $R^t$ be the corresponding reward.
		For $\epsilon > 0$, the bandit model is said to be $\epsilon$-LDP if i) there is an $\epsilon$-DP mechanism $M: \mathcal{D} \rightarrow \mathbb{R}$ and ii) $A^{t+1} \in \sigma(A^s, M(R^s): 1\leq s \leq t)$ for any time $t$.
	\end{definition}

\subsection{Related work}


	Non-private MAB problems have been studied for decades. For non-private bandit problems, either frequentist methods like UCB (Upper Confidence Bound) \cite{FiniteAnalysis2002} or Bayesian methods like Thompson Sampling \cite{ThompsonSampling2012} have been shown to achieve optimal regret performance (up to constant factors). For a literature review on MAB, we refer readers to \cite{BanditSurvey2018}. Recently, privacy issues have received increasing attention in the machine learning community. Differential privacy \cite{DPFoundation2014} provides a quantitative metric to measure the privacy level and has been gaining popularity. We refer readers to \cite{DPFoundation2014} that introduces the fundamental concepts and methods of DP. 

	To the best of our knowledge, the earliest work that studied LDP bandits is \cite{DPBanditsCorrupt2018}, which proposed an LDP bandit algorithm that works for arms with Bernoulli rewards. In comparison, our algorithms can work for a much more general set of instances. The other work that studied the LDP bandit problem is \cite{DPBanditLowerBound2019}, in which distribution-dependent and distribution-free regret lower bounds were proved. We note that the distribution-dependent regret lower bound in \cite{DPBanditLowerBound2019} is looser than the one proved in this paper. 
	
	Besides the LDP bandit mode, there have been other works on MAB problems with other types of DP guarantees \cite{DPBanditMultiPartyContextual2019,DPBanditEmpirical2019,DPBanditWithContexual2014,DPBandit2015,DPBandit2019,DPBanditContextualSequential2018,DPBanditSequentialMultiAgent2015,DPBanditSequential2016,DPBanditAdcersarial2017}. These works are not directly comparable to our work, and so we give a brief introduction here. In the bandit models of \cite{DPBanditWithContexual2014,DPBandit2015,DPBandit2019}, it is difficult for the agent to learn individual rewards from the private empirical means, i.e., the curator can aggregate the rewards from different users. In the bandit models of \cite{DPBanditContextualSequential2018,DPBanditSequentialMultiAgent2015,DPBanditSequential2016}, it is difficult for any adversary to learn the individual rewards from the sequence of actions taken by the agent, i.e., the agent is trusted. This model is named the sequential DP bandits in \cite{DPBanditLowerBound2019}. In the bandit models of \cite{DPBanditMultiPartyContextual2019}, it is difficult for any adversary to learn the context features in a contextual bandit setting, which we name it as environmental DP bandits. In \cite{DPBanditAdcersarial2017}, the authors studied privacy-preserving adversarial bandits.

\subsection{Main results}
Our key contributions are summarized as follows: 
\begin{itemize}
	\item We prove a tight regret lower bound (up to a constant factor) for the LDP bandit problem. 
	\item For bandits with bounded rewards, we propose a Laplace mechanism and a Bernoulli mechanism, and develop corresponding UCB algorithms for them, both of which match the lower bound proved in this paper (up to constant factors). 
	\item For bandits with unbounded support and i.i.d. \footnote{Term ``i.i.d.'' stands for ``identically independent distributed''.} sub-Gaussian noises, we use a Sigmoid preprocessing and obtain algorithms with \textit{tight} regret upper bounds (up to constant factors).
\end{itemize}

\section{Lower bound}
	In this section, we present the regret lower bound of LDP bandit algorithms. Typically, the lower bound of regret minimization depends on the KL-divergence \cite{InformationTheory2012} between the latent distributions of the optimal arm and suboptimal arms \cite{KLLowerBound1985}. Let $f$ and $g$ be the probability density function (PDF) of two distributions, and we allow point masses in PDFs. The KL-divergence between $f$ and $g$ is defined as $D_{\mbox{\tiny{KL}}}(f||g) := \int_{\mathbb{R}}f(x)\log[f(x)/g(x)]\ \mathrm{d}x$, where we stipulate that $0\cdot \log{0} = 0$.\footnote{All $\log$ is this paper are natural $\log$. $0\cdot\log{0}$ is because $\lim_{x\rightarrow 0^+}x\log{x} = 0$.} In \cite{DPBanditLowerBound2019}, the authors proved a lower bound for LDP bandit algorithms that depends on the KL-divergence between latent distributions. However, in this paper, our Theorem~\ref{Thm:LB} states a tighter lower bound that depends on the values of $\Delta_a = \mu^* - \mu$ but not the KL-divergences. Due to space limitation, we leave the proof of Theorem~\ref{Thm:LB} to the supplementary material. Later, Theorems~\ref{Thm:LDP-UCB-L} and \ref{Thm:LDP-UCB-B} will show that for $\epsilon \leq 1$, this lower bound is tight in order sense. 
	\begin{restatable}[Lower bound]{theorem}{RestateLB}
		\label{Thm:LB}
		Let $\epsilon > 0$ be given. Assume that the rewards of all arms follow Bernoulli distributions. The regret $R(T)$ of any $\epsilon$-LDP policy satisfies
		\begin{align}
			\liminf_{T\rightarrow \infty}\frac{R(T)}{\log{T}} \geq \frac{1}{(e^\epsilon - e^{-\epsilon})^2} \sum_{a:\Delta_a > 0}\frac{1}{\Delta_a}.\nonumber
		\end{align}
		When $\epsilon \rightarrow 0$, since $e^\epsilon - e^{-\epsilon} \simeq 2\epsilon$, we have $\liminf_{T\rightarrow\infty}\frac{R(T)}{\log{T}} \gtrsim \frac{1}{4\epsilon^2}\sum_{a:\Delta_a > 0}\frac{1}{\Delta_a}$.
	\end{restatable}


\section{Algorithms and upper bounds}

\subsection{Mechanisms for MAB with bounded rewards}
	In this section, we propose two $\epsilon$-DP mechanisms: one is to convert bounded rewards to Laplace responses, and the other is to convert bounded rewards to Bernoulli responses. After converting the rewards to private responses, the agent uses UCB-like methods, e.g., \cite{FiniteAnalysis2002}, to trade off the exploration and exploitation. Although the agent has no access to the actual rewards of the arms, it can bound the number of pulls of any suboptimal arm by similar techniques as the non-private UCB algorithms. Based on both mechanisms, the private UCB algorithms can achieve optimal regrets (up to constant factors). In this paper, we adopt the Hoeffding bounds in \cite{FiniteAnalysis2002} to bound the empirical mean rewards. We note that one may use other confidence bounds by which one may get better constant factors, but this is beyond the scope of this paper. From the theoretical perspective, the Hoeffding bounds in \cite{FiniteAnalysis2002} can already achieve optimal regrets in order sense. 

\subsubsection{Laplace mechanism}
	The Laplace mechanism \cite{DPFoundation2014} (i.e., adding Laplace noises to data records) is a widely used mechanism in the areas of DP. The key idea of the Laplace mechanism is to add an independent Laplace$(\frac{1}{\epsilon})$ noise to each reward, which preserves $\epsilon$-DP and does not change the mean values of the records. For any $b > 0$, the PDF of the Laplace$(b)$ distribution is defined as:
	\begin{align}
		\mbox{Laplace}(b): l(x\mid b) = (2b)^{-1}\exp(-|x|/{b}). \nonumber
	\end{align}
	The mean of Laplace$(b)$ distribution is $0$, and its variance is $2b^2$.
	The Laplace mechanism is stated in Curator~\ref{Curator:CTL} and its theoretical guarantee is stated in Lemma~\ref{Lm:CTL}. 

	\makeatletter
	\renewcommand{\ALG@name}{Curator}
	\makeatother
	\begin{algorithm}[h]
		\caption{Convert-to-Laplace$(\epsilon)$ (CTL$(\epsilon)$)}\label{Curator:CTL}
		\textbf{On receiving} a reward $r$ from the user:
		\begin{algorithmic}
			\State \Return $M_L(r) = r + L$, where $L\sim$ Laplace$(1/\epsilon)$ distribution;
		\end{algorithmic}
	\end{algorithm}

	\begin{restatable}[\textbf{Proposition~3.3 in \cite{LaplaceDP2016}}]
		{lemma}{RestateLaplace}\label{Lm:CTL}
		Curator CTL (i.e., $M_L$) is $\epsilon$-DP on $[0,1]$.
	\end{restatable}

	With CTL, we develop a UCB algorithm that takes the private responses of CTL as the input. Note that since CTL adds a Laplace noise to each reward, we need an additional term to bound the summations of independent Laplace random values. We adopt the concentration inequality used in \cite{FiniteAnalysis2002}, which is stated in Lemma~\ref{Lm:LaplaceConcentration}. 

	\begin{restatable}[\textbf{Lemma~2.8 in \cite{LaplaceConcentration2011}}]{lemma}{RestateLaplaceConcentration}\label{Lm:LaplaceConcentration}
		Let $X_1,X_2,...,X_n$ be i.i.d. random variables following the Laplace$(b)$ distribution and $Y_n = X_1 + X_2 + \cdots + X_n$. For $v \geq b\sqrt{n}$ and $0 < \lambda < \frac{2\sqrt{2}v^2}{b}$, we have $\mathbb{P}\{Y > \lambda\} \leq \exp(-\frac{\lambda^2}{8v^2})$. 
	\end{restatable}

	The corresponding UCB algorithm, termed LDP-UCB-L (LDP UCB algorithm with the Laplace mechanism), is described in Agent~\ref{Agent-L}. In Line~3, the term $\sqrt{({2\log{t}})/{N^t_a}}$ is the Hoeffding bound for bounding the summation of independent bounded random variables and the term $\sqrt{(32\log{t})/(N^t_a\epsilon^2)}$ is for bounding the summation of independent Laplace variables, which is derived from Lemma~\ref{Lm:LaplaceConcentration}. The theoretical guarantee of LDP-UCB-L is stated in Theorem~\ref{Thm:LDP-UCB-L} and the proof is relegated to the supplementary material due to space limitation.

	\makeatletter
	\renewcommand{\ALG@name}{Agent}
	\makeatother
	\begin{algorithm}[h]
		\caption{LDP-UCB-L$(\epsilon)$ (LDP UCB algorithm with Laplace mechanism)}\label{Agent-L}
		\begin{algorithmic}[1]
			\State Pull each arm once and receive the private responses from CTL$(\epsilon)$; $t\gets n$;
			\State Define $\hat{\mu}^t_a:= $empirical mean of the private responses of arm $a$ till time $t$;
			\State Define $N^t_a :=$ number of pulls of arm $a$ and $u^t_a := \hat{\mu}^t_a + \sqrt{({2\log{t}})/{N^t_a}} + \sqrt{(32\log{t})/(\epsilon^2N^t_a)}$;
			\While{$t < T$}
				\If{there is an arm $a$ such that $N^t_a \leq 4\log{(t+1)}$} 
					\State $a^t \gets a$;\quad\#To satisfies Lemma~\ref{Lm:LaplaceConcentration}'s requirement, details in the proof of Theorem~\ref{Thm:LDP-UCB-L}
				\Else \quad $a^t \gets \argmax_{a\in[n]}u^t_a$;
				\EndIf
				\State Pull arm $a^t$ once and receive the private response from CTL$(\epsilon)$;
				\State  $t \gets t + 1$; Update $\hat{\mu}^t_a$, $N^t_a$, and $u^t_a$ for arms $a$;
			\EndWhile
		\end{algorithmic}
	\end{algorithm}

\begin{restatable}{theorem}{RestateLDPUCBL}
	\label{Thm:LDP-UCB-L}
	LDP-UCB-L is $\epsilon$-LDP. Its distribution-dependent regret is at most
	\begin{align}
	\sum_{a:\Delta_a>0}\Big[\frac{8(1 + 4/\epsilon)^2\log{T}}{\Delta_a} + \Big(1 + \frac{2\pi^2}{3}\Big)\Delta_a\Big] = 	O\Big(\sum_{a:\Delta_a>0}\Big[\frac{\log{T}}{\epsilon^2\Delta_a} + \Delta_a \Big]\Big), \nonumber
	\end{align}
	and its distribution-free regret (for $T \geq n$) is at most $O(\epsilon^{-1}\sqrt{nT\log{T}})$. 
\end{restatable}
\textbf{Remark.} 
	i) Compared to non-private UCB using the same confidence bounds \cite{FiniteAnalysis2002}, the regret of LDP-UCB-L is increased by a $(1+4/\epsilon)^2$ factor, which can be viewed as the cost for preserving privacy. When $\epsilon$ approaches infinity, this factor approaches one, and the regret approaches that of the non-private version. ii) According to Theorem~\ref{Thm:LB}, the distribution-dependent regret of LDP-UCB-L is optimal (up to a constant factor). iii) In \cite[Theorem~1]{DPBanditLowerBound2019}, a distribution-free lower bound $\Omega(\epsilon^{-1}\sqrt{nT})$ was given, and thus the distribution-free regret of LDP-UCB-L is optimal up to a $\sqrt{\log{T}}$ factor. 
	
	If we change ``for any $x$ and $x'$ in $\mathcal{D}$'' to ``for any $x$ and $x'$ with $||x-x'||_1\leq 1$'', then CTL$(\epsilon)$ is $\epsilon$-DP on $\mathbb{R}$ \cite{DPFoundation2014}. Thus, by changing the terms $\sqrt{(2\log{t})/N^t_a}$ to proper confidence bounds, CTL and LDP-UCB-L are still $\epsilon$-DP or $\epsilon$-LDP for bandit instances without a bounded support.

\subsubsection{Bernoulli mechanism}
	In addition to the Laplace mechanism, we propose another mechanism called Convert-to-Bernoulli (CTB), which converts bounded rewards to Bernoulli responses. Both the theoretical analysis and the empirical results indicate that the Bernoulli mechanism performs better than the Laplace mechanism.\footnote{However, if we can find tighter concentration bounds on the summation of independent Laplace variables, then we may get better regret bounds for the Laplace mechanism.}
	
	In \cite{DPBanditsCorrupt2018}, the authors proposed a similar mechanism that only works for Bernoulli rewards. By contrast, in this paper, we allow the reward to be an arbitrary value in $[0,1]$. CTB is described in Curator~\ref{Curator:CTB}. Its theoretical guarantee is stated in Lemma~\ref{Lm:CTB}, and the proof is left to the supplementary material.
	
	\makeatletter
	\renewcommand{\ALG@name}{Curator}
	\makeatother
	\begin{algorithm}[h]
	\caption{Convert-to-Bernoulli$(\epsilon)$ (CTB)}\label{Curator:CTB}
	\textbf{On receiving} a reward $r \in [0,1]$ from the user:
		\begin{algorithmic}
			\State \Return $M_B(r) = $ an independent sample of Bernoulli$(\frac{re^\epsilon + 1 - r}{1 + e^\epsilon})$;
		\end{algorithmic}
	\end{algorithm}

\begin{restatable}
	{lemma}{RestateCTB}\label{Lm:CTB}
	Curator CTB (i.e., $M_B$) is $\epsilon$-DP on $[0,1]$, and the returned value follows the Bernoulli distribution with mean $\mu_{a,\epsilon} := \frac{1}{2} + (2\mu_a - 1)\cdot \frac{e^\epsilon - 1}{2(e^\epsilon + 1)}$.
\end{restatable}

	We can view CTB as a procedure that converts an arm $a$ to a Bernoulli arm with mean $\mu_{a,\epsilon}$. By CTB, we take the converted arms as inputs to non-private UCB algorithms, and obtain an LDP UCB algorithm called LDP-UCB-B (LDP UCB algorithm with the Bernoulli mechanism), which is described in Agent~\ref{Agent-B}. By similar insights as in the non-private UCB algorithms, we can bound the number of pulls of each suboptimal arm, and hence, upper bound the regret. The theoretical guarantee is stated in Theorem~\ref{Thm:LDP-UCB-B} and the proof is relegated to the supplementary material.
	
	\makeatletter
	\renewcommand{\ALG@name}{Agent}
	\makeatother
	\begin{algorithm}[h]
		\caption{LDP-UCB-B$(\epsilon)$ (LDP UCB algorithm with Bernoulli mechanism)}\label{Agent-B}
		\begin{algorithmic}[1]
			\State Pull each arm once and receive the private responses from CTB$(\epsilon)$; $t\gets n$;
			\State $\hat{\mu}^t_a:= $empirical mean of the private responses of arm $a$ till time $t$;
			\State $N^t_a :=$ number of pulls of arm $a$; $u^t_a := \hat{\mu}^t_a + \sqrt{({2\log{t}})/{N^t_a}}$;
			\While{$t < T$}
				\State $a^t \gets \argmax_{a\in[n]}u^t_a$;
				\State Pull arm $a^t$ once and receive the private response from CTB$(\epsilon)$; 
				\State $t\gets t + 1$; Update $\hat{\mu}^t_a$, $N^t_a$, and $u^t_a$ for arms $a$;
			\EndWhile
		\end{algorithmic}
	\end{algorithm}
	
	\begin{restatable}[Theoretical guarantee of LDP-UCB-B]{theorem}{RestateLDPUCBB}
		\label{Thm:LDP-UCB-B}
		LDP-UCB-B$(\epsilon)$ is $\epsilon$-LDP. Its distribution-dependent regret is at most
		\begin{align}
			\sum_{a:\Delta_a>0}\Big[\frac{8}{\Delta_a}\Big(\frac{e^\epsilon+1}{e^\epsilon-1}\Big)^2\log{T} + \Big(1 + \frac{\pi^2}{3}\Big)\Delta_a\Big] = 	O\Big(\sum_{a:\Delta_a>0}\Big[\frac{\log{T}}{\epsilon^2\Delta_a} + \Delta_a \Big]\Big), \nonumber
		\end{align}
		and its distribution-free regret (for $T \geq n$) is at most $O(\epsilon^{-1}\sqrt{nT\log{T}})$.
	\end{restatable}
	\textbf{Remark.}
	i) Compared to non-private UCB algorithms using the same confidence bounds \cite{FiniteAnalysis2002}, the regret of LDP-UCB-L is increased by a $(\frac{e^\epsilon+1}{e^\epsilon-1})^2$ factor, which can be viewed as the cost for preserving privacy. When $\epsilon$ approaches infinity, this factor approaches one, and the regret approaches that of the non-private version. ii) According to Theorem~\ref{Thm:LB}, the distribution-dependent regret of LDP-UCB-B is optimal (up to a constant factor). iii) In \cite[Theorem~1]{DPBanditLowerBound2019}, a distribution-free lower bound $(\Omega(\epsilon^{-1}\sqrt{nT}))$ was given, and thus the distribution-free regret of LDP-UCB-B is optimal up to a $\sqrt{\log{T}}$ factor. iv) The $\epsilon$-term $(\frac{e^\epsilon + 1}{e^\epsilon -1})^2$ of LDP-UCB-B is always smaller than $(1+\frac{4}{\epsilon})^2$, that of LDP-UCB-L. When $\epsilon$ increases, the difference becomes smaller, and when $\epsilon$ approaches infinity, they all converge to one. Later, the numerical results will also indicate that LDP-UCB-B's empirical performance tends to be better than that of LDP-UCB-L and the difference tends to be smaller as $\epsilon$ increases. 

\subsection{Mechanisms for MAB with unbounded reward supports}
	In practice, the rewards may not have bounded supports, and the mechanisms studied in the last subsection do not work in this situation. For the Laplace mechanism, adding Laplace$(s/\epsilon)$ noise to the rewards does not provide $\epsilon$-DP if the difference between two rewards is larger than $s$.
	For the Bernoulli mechanism, when $r < 0$ or $r > 1$, the value $(re^\epsilon+1-r)/(e^\epsilon + 1)$ is outside of $[0,1]$, making the Bernoulli mechanism ill-defined. 
	
	To deal with unbounded rewards, we first map the rewards with a Sigmoid function, for which the outputs are guaranteed to be in $[0,1]$. Sigmoid function is defined as $s(r) := (1 + e^{-r})^{-1}$ for any $r \in \mathbb{R}$, which happens to guarantee that the gap between the expected mapped rewards of any two arms $a$ and $b$ is lower bounded by $\Omega(|\mu_a-\mu_b|)$ for bandit instances with i.i.d. sub-Gaussian noises. This property is stated in Lemma~\ref{Lm:Sigmoid}. Other logistic functions may also have similar properties, but in this paper we focus on Sigmoid.
	
	\begin{restatable}{lemma}{RestateSigmoid}\label{Lm:Sigmoid}
		Let $0 \leq \mu \leq \lambda \leq 1$ and $\mathcal{N}$ be a sub-Gaussian distribution with mean zero and variance one. For $s(r) = (1+e^{-r})^{-1}$, $X = \lambda + Z_1$ and $Y = \mu + Z_2$, where $Z_1$ and $Z_2$ are i.i.d. in $\mathcal{N}$, we have $\mathbb{E}[s(X) - s(Y)] \geq c_s(\lambda - \mu)$, where $c_s > 0$ is a universal constant.
	\end{restatable}

	To make the above lemma hold, in this subsection, we make further assumptions on the bandit model. Assume that mean rewards, $\mu_1, \mu_2,...,\mu_n$, are bounded and have been rescaled to $[0,1]$. Also, for each arm $a$ and time $t$, we assume that the reward of the $t$-th pull of arm $a$ is $\mu_a + Z_a^t$, where $(Z_a^t,a\in[n],t\in\mathbb{Z}^+)$ are i.i.d. sub-Gaussian with mean zero and variance at most one.

	After the Sigmoid mapping, the new mechanism maps the Sigmoid values to CTL or CTB. We name these two new mechanism as CTL-S (CTL with Sigmoid preprocessing) and CTB-S (CTB with Sigmoid preprocessing). They are described in Curators~\ref{Curator:CTL-S} and \ref{Curator:CTB-S}, respectively, and their theoretical guarantees are stated in Lemmas~\ref{Lm:CTL-S} and \ref{Lm:CTB-S}. With these two new mechanisms, we develop new LDP UCB algorithms LDP-UCB-LS (LDP-UCB-L with Sigmoid preprocessing) and LDP-UCB-BS (LDP-UCB-B with Sigmoid preprocessing) for bandits with unbounded reward supports, whose theoretical guarantees are stated in Corollaries~\ref{Thm:LDP-UCB-LS} and \ref{Thm:LDP-UCB-BS}, respectively. 

	\makeatletter
	\renewcommand{\ALG@name}{Curator}
	\makeatother
	
	\begin{algorithm}[h]
		\caption{Convert-to-Laplace-Sigmoid$(\epsilon)$ (CTL-S)}\label{Curator:CTL-S}
		\textbf{On receiving} a reward $r$ from the user:
		\begin{algorithmic}
			\State \Return $M_{LS}(r) = (1+e^{-r})^{-1} + L$, where $L\sim$ Laplace$(1/\epsilon)$ distribution. 
		\end{algorithmic}
	\end{algorithm}

	\begin{algorithm}[h]
		\caption{Convert-to-Bernoulli-Sigmoid$(\epsilon)$ (CTB-S)}\label{Curator:CTB-S}
		\textbf{On receiving} a reward $r$ from the user:
		\begin{algorithmic}
			\State \Return $M_{BS}(r) = $ an independent sample of Bernoulli$(\frac{s(r)e^\epsilon + 1 - s(r)}{1 + e^\epsilon})$, where $s(r) = (1 +e^{-r})^{-1}$;
		\end{algorithmic}
	\end{algorithm}

	\begin{restatable}
	{lemma}{RestateCTLS}\label{Lm:CTL-S}
		Curator CTL-S (i.e., $M_{LS}$) is $\epsilon$-DP. For two arms $a$ and $b$ with mean rewards $\mu_a \geq \mu_b$, the difference between the expected responses of CTL-S$(\epsilon)$ on arms $a$ and $b$ is at least $c_s(\mu_{a}-\mu_{a})$, where $c_s > 0$ is a universal constant.
	\end{restatable}

	\begin{restatable}
		{lemma}{RestateCTBS}\label{Lm:CTB-S}
		Curator CTB-S (i.e., $M_{BS}$) is $\epsilon$-DP. For two arms $a$ and $b$ with mean rewards $\mu_a \geq \mu_b$, the difference between the expected responses of CTB-S$(\epsilon)$ on arms $a$ and $b$ is at least $c_s(\mu_{a,\epsilon}-\mu_{a,\epsilon})$, where $c_s > 0$ is a universal constant.
	\end{restatable}

	\begin{restatable}
		{corollary}{RestateLDPUCBLS}\label{Thm:LDP-UCB-LS}
		Replacing CTL in LDP-UCB-L by CTL-S, we get LDP-UCB-LS. LDP-UCB-LS is $\epsilon$-LDP. Its distribution-dependent regret is at most
		\begin{align}
			\frac{1}{c_s^2}\sum_{a:\Delta_a>0}\Big[\frac{8(1 + 4/\epsilon)^2\log{T}}{\Delta_a} + \Big(1 + \frac{2\pi^2}{3}\Big)\Delta_a\Big] = 	O\Big(\sum_{a:\Delta_a>0}\Big[\frac{\log{T}}{\epsilon^2\Delta_a} + \Delta_a \Big]\Big), \nonumber
		\end{align}
		where $c_s>0$ is a universal constant. Its distribution-free regret is at most $O(\epsilon^{-1}\sqrt{nT\log{T}})$.
	\end{restatable}

	\begin{restatable}
		{corollary}{RestateLDPUCBBS}\label{Thm:LDP-UCB-BS}
		Replacing CTB in LDP-UCB-B by CTB-S, we get LDP-UCB-BS. LDP-UCB-BS is $\epsilon$-LDP. Its distribution-dependent regret is at most
		\begin{align}
			\frac{1}{c_s^2}\sum_{a:\Delta_a>0}\Big[\frac{8}{\Delta_a}\Big(\frac{e^\epsilon+1}{e^\epsilon-1}\Big)^2\log{T} + \Big(1 + \frac{\pi^2}{3}\Big)\Delta_a\Big] = 	O\Big(\sum_{a:\Delta_a>0}\Big[\frac{\log{T}}{\epsilon^2\Delta_a} + \Delta_a \Big]\Big), \nonumber
		\end{align}
		where $c_s > 0$ is a universal constant. Its distribution-free regret is at most $O(\epsilon^{-1}\sqrt{nT\log{T}})$.
	\end{restatable}

%
	
\section{Numerical results}
	In this section, we illustrate the numerical results for our algorithms. Due to space limitation, we only present the results for bandits with bounded supports. The results for bandits with unbounded supports can be found in the supplementary material. To the best of our knowledge, there is no previous LDP bandit algorithm in the literature except in \cite{DPBanditsCorrupt2018}. However, the algorithm in \cite{DPBanditsCorrupt2018} only works for Bernoulli rewards and can be viewed as a special case of our Bernoulli mechanism. Thus, the only LDP bandit algorithms we present are LDP-UCB-L and LDP-UCB-B. We also include the performance of the non-private UCB algorithm (i.e., $\epsilon = \infty$) as a baseline to see the cost for preserving $\epsilon$-LDP. Here, we use the UCB1 algorithm in \cite{FiniteAnalysis2002} as the baseline since our private algorithms adopt the same confidence bounds as in \cite{FiniteAnalysis2002}. The codes can be found in the supplementary material.
	
	The numerical results are illustrated in Figure~\ref{fig:bounded}. In all the experiments, we set the number of arms $n = 20$. The best arm has a mean reward $0.9$; five arms have mean rewards $0.8$; five arms have mean rewards $0.7$; five arms have mean rewards $0.6$; and four arms have mean rewards $0.5$. In Figure~\ref{fig:bounded} (a) to (d), we use Bernoulli arms, i.e., the rewards of all arms follow Bernoulli distributions. In Figure~\ref{fig:bounded} (e) (f), the rewards of arms follow different types of distributions to show that our algorithms work beyond Bernoulli arms. To be specific, arms with mean rewards $0.9$ or $0.6$ generate rewards from Bernoulli distributions; arms with mean rewards $0.8$ generate rewards from Beta$(4,1)$ distribution; arms with mean rewards $0.7$ generate rewards from $\{0.4,1\}$ uniformly at random; and arms with mean rewards $0.5$ generate rewards from $[0,1]$ uniformly at random. Each line in each figure is averaged over 50 independent trials.

	\begin{figure}[h]\centering
		\begin{subfigure}[b]{0.32\textwidth}
			\includegraphics[scale=0.25]{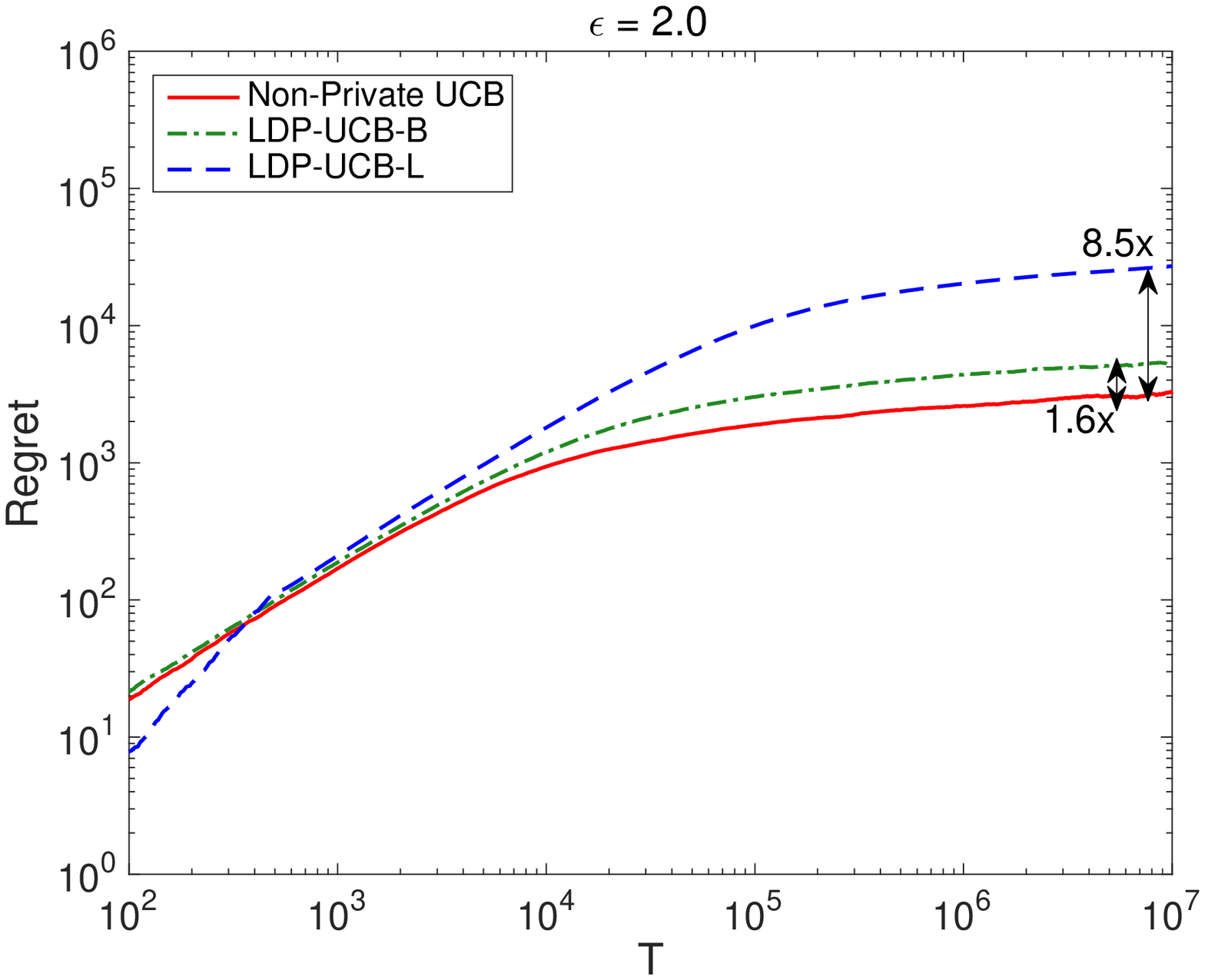}
			\caption{Bernoulli arms, $\epsilon = 2.0$.}
		\end{subfigure}\ \ 
		\begin{subfigure}[b]{0.32\textwidth}
			\includegraphics[scale=0.25]{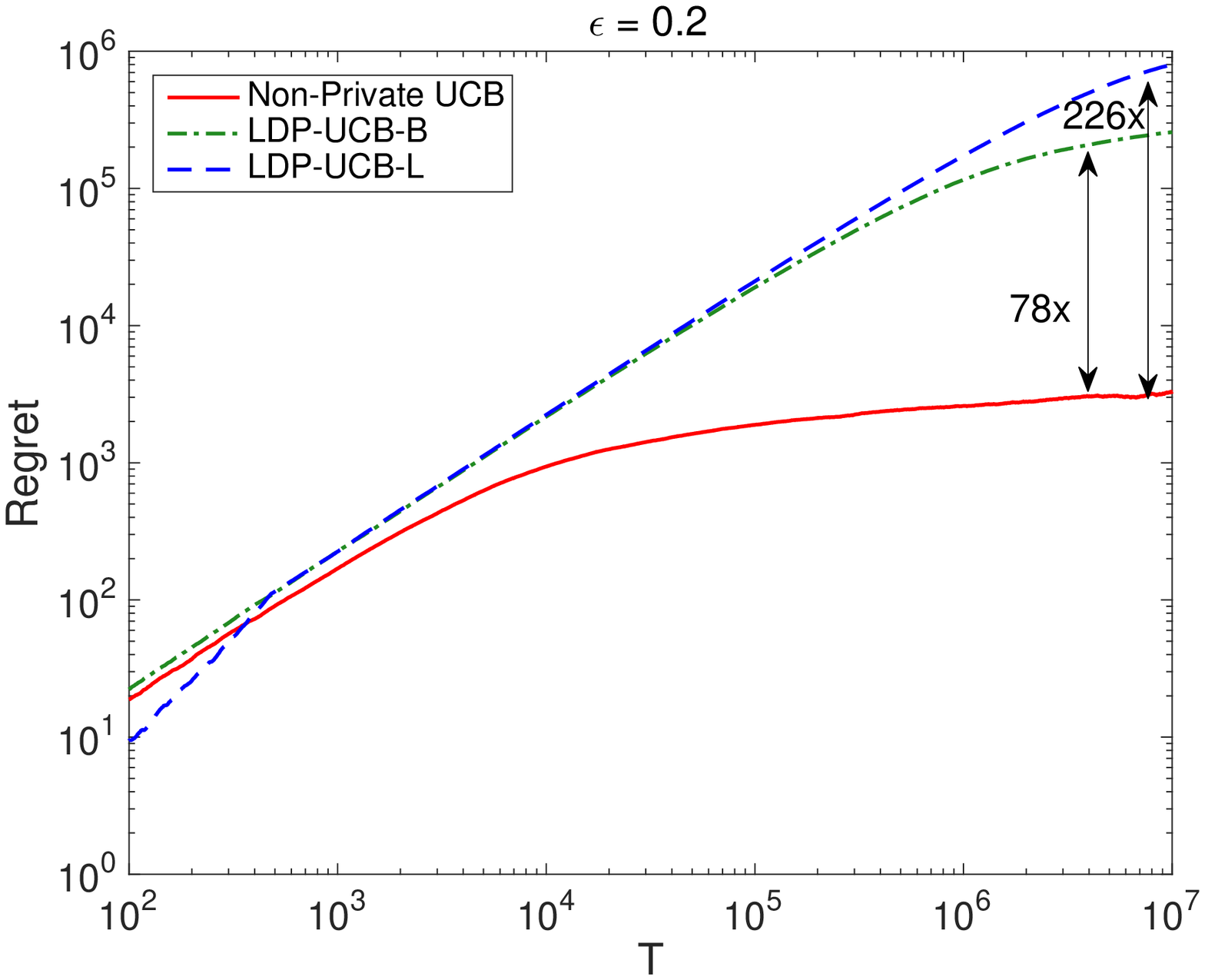}
			\caption{Bernoulli arms, $\epsilon = 0.2$.}
		\end{subfigure}\ \ 
		\begin{subfigure}[b]{0.32\textwidth}
			\includegraphics[scale=0.25]{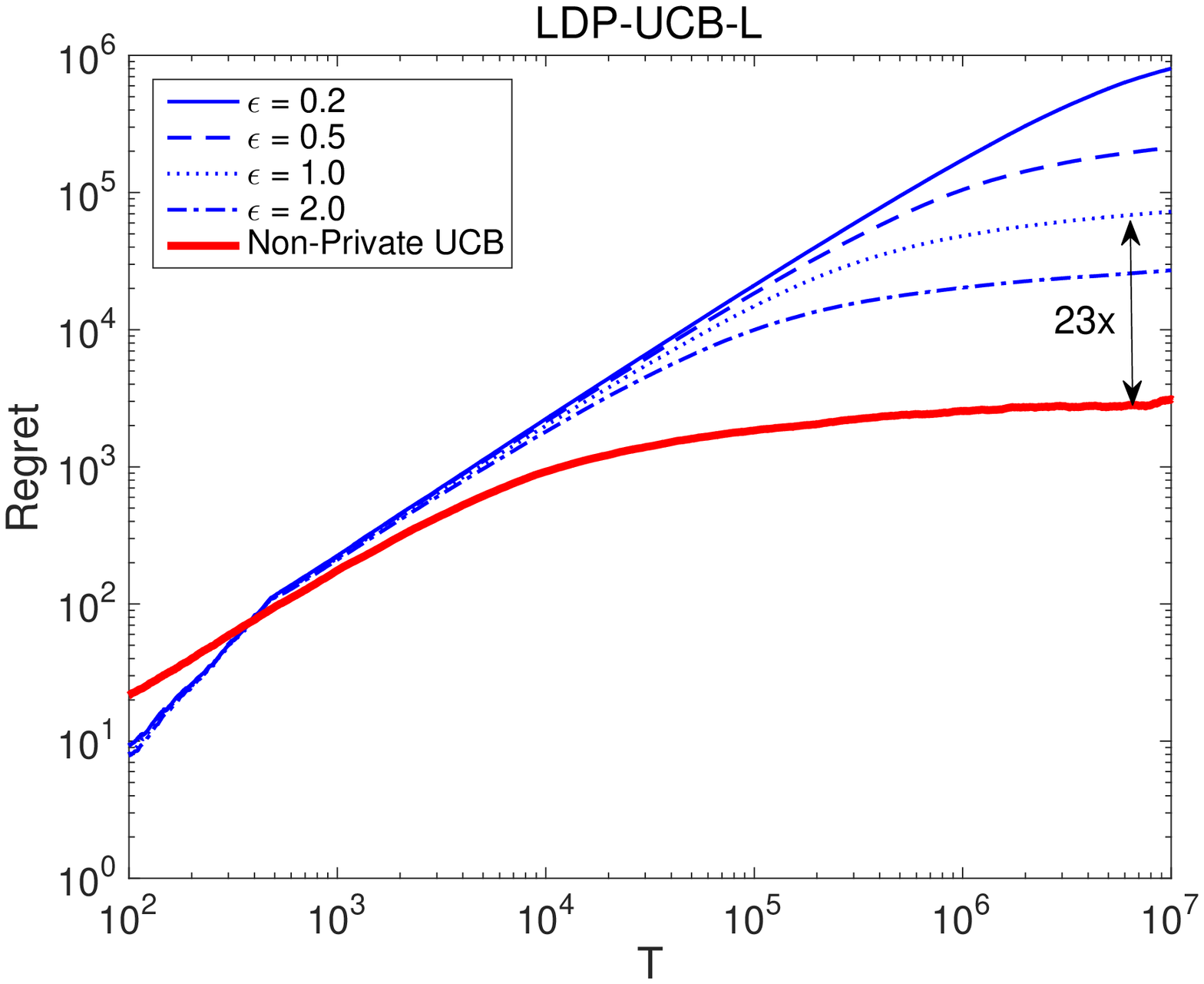}
			\caption{LDP-UCB-L, vary $\epsilon$.}
		\end{subfigure}\\
		\begin{subfigure}[b]{0.32\textwidth}
			\includegraphics[scale=0.25]{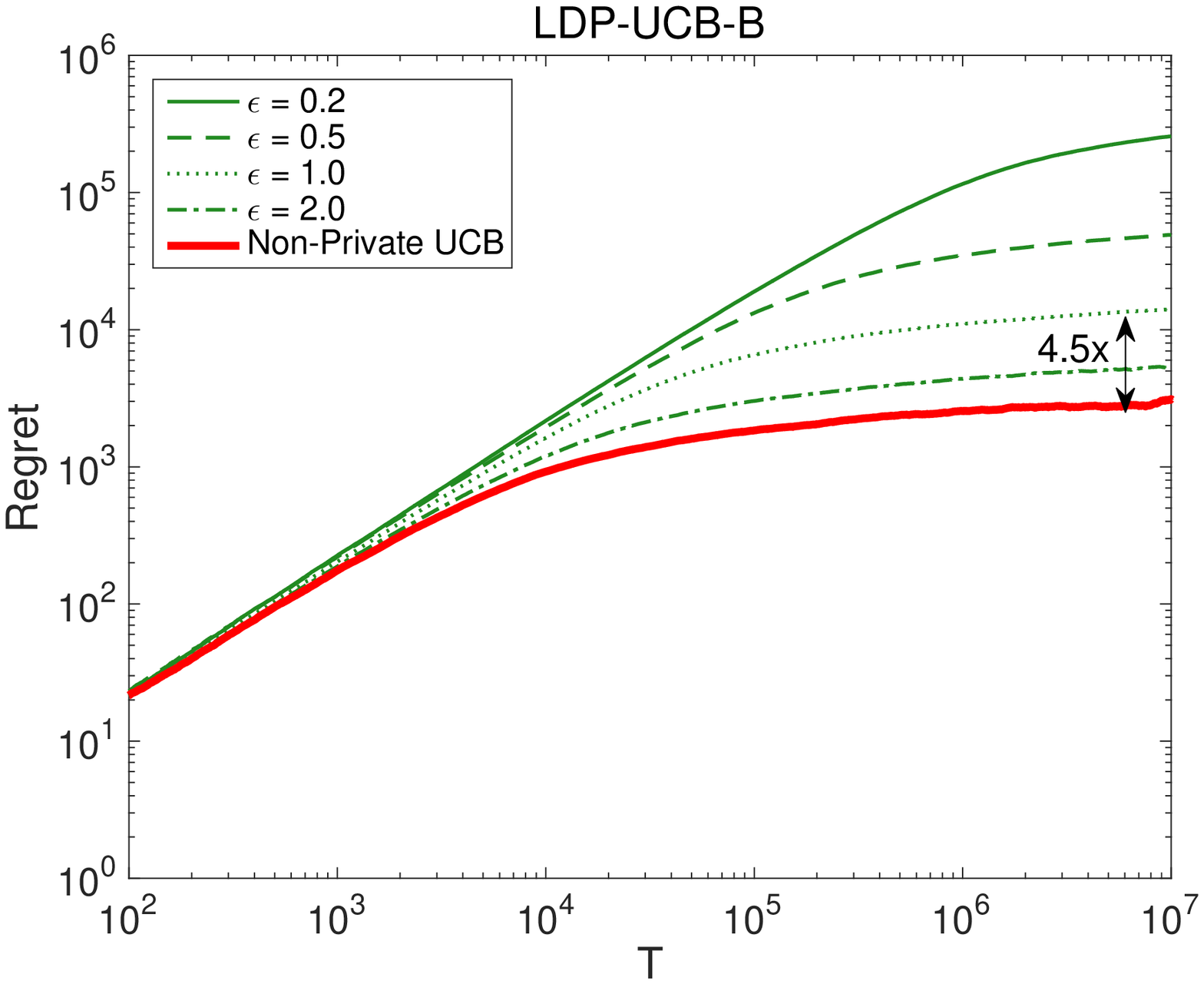}
			\caption{LDP-UCB-B, vary $\epsilon$.}
		\end{subfigure}\ \ 
		\begin{subfigure}[b]{0.32\textwidth}
			\includegraphics[scale=0.25]{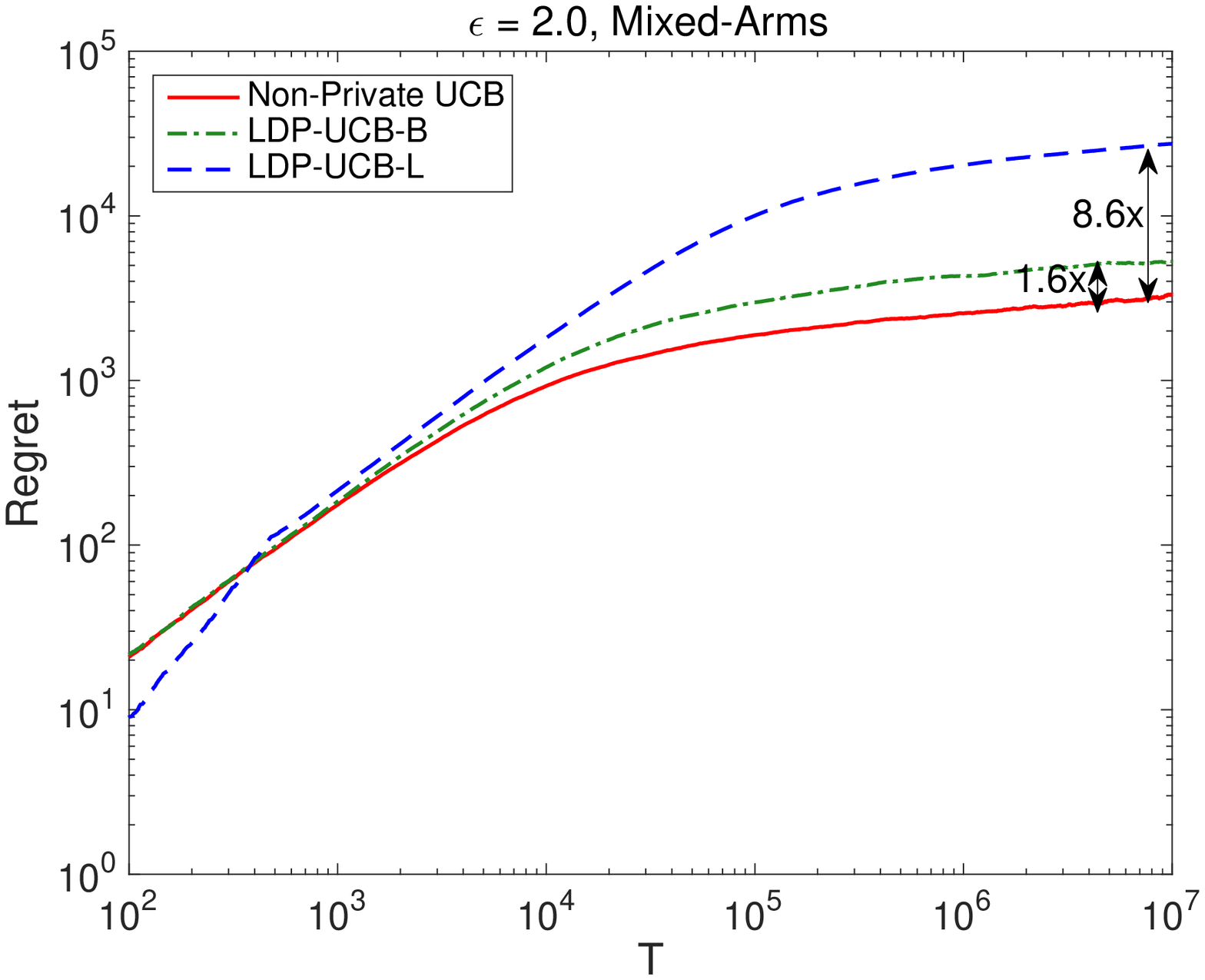}
			\caption{Mixed types of arms, $\epsilon = 2.0$.}
		\end{subfigure}\ \ 
		\begin{subfigure}[b]{0.32\textwidth}
			\includegraphics[scale=0.25]{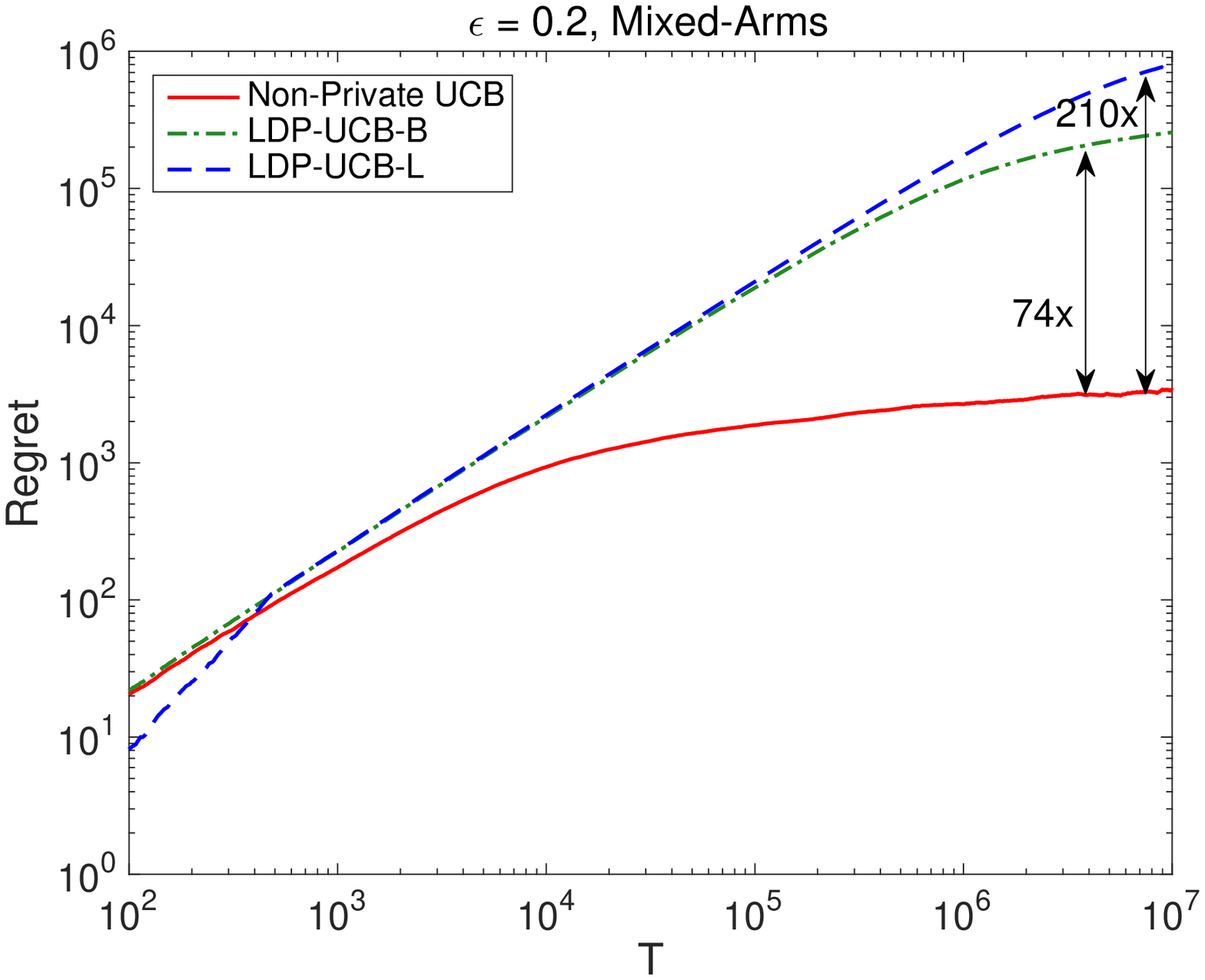}
			\caption{Mixed types of arms, $\epsilon = 0.2$.}
		\end{subfigure}
		\caption{Numerical results for LDP-UCB-L and LDP-UCB-B.}\label{fig:bounded}
	\end{figure}
	
	In Figure~\ref{fig:bounded} (a), we fix $\epsilon = 2.0$. We can see that the regret of LDP-UCB-B is slightly larger than that of the non-private UCB and smaller than that of LDP-UCB-L. The ratio of the regrets of LDP-UCB-B to non-private UCB is $1.6$, and that of LDP-UCB-L is $8.5$. In theory, the upper bounds of the ratios are $(\frac{e^\epsilon+1}{e^\epsilon-1})^2 = 1.7$ for LDP-UCB-B and $(1+\frac{4}{\epsilon})^2 = 9.0$ for LDP-UCB-L. Thus, the numerical results in Figure~\ref{fig:bounded} (a) are consistent with our theoretical results. In Figure~\ref{fig:bounded} (b), we fix $\epsilon = 0.2$. The ratio of the regret of LDP-UCB-L (LDP-UCB-B) to non-private UCB becomes much larger, which is consistent with the theory that the ratio grows with $(1+\frac{4}{\epsilon})^2$ ($(\frac{e^\epsilon+1}{e^\epsilon-1})^2$). In theory, the ratios are upper bounded by $441$ and $101$, respectively, which are larger and not far away from the empirical results. 
	
	In Figure~\ref{fig:bounded} (c), we compare the regrets of LDP-UCB-L with different $\epsilon$-values. In Figure~\ref{fig:bounded} (d), we do the same for LDP-UCB-B. From (c) and (d), we can see that the regrets of LDP-UCB-L and LDP-UCB-B both increase with $\frac{1}{\epsilon}$ and the convergence speed both decrease as $\epsilon$ decreases. In Figure~\ref{fig:bounded} (e) and (f), the rewards of the arms follow different types of rewards, and the performances of LDP-UCB-L and LDP-UCB-B are similar to that for Bernoulli arms, which indicates that our algorithms work for arms with various types of latent distributions. 
	
\section{Conclusion}
	This paper studied the multi-armed bandit problem with local differential privacy guarantee. We proved the tight regret lower bound and proposed algorithms with tight regret upper bounds (up to constant factors). Numerical results also confirmed our theoretical results.

\bibliography{dpbandit}

\clearpage

\begin{appendices}
	{\LARGE Supplementary material}
	
\section{Proofs}
\subsection{Proof of Theorem~\ref{Thm:LB}}
\RestateLB*
	\begin{proof}[\textbf{Proof.}]
		Let $n$ Bernoulli arms with means $\mu_1, \mu_2,...,\mu_n$ be given. Let $M : [0,1] \rightarrow \mathbb{R}$ be an arbitrary $\epsilon$-DP randomized mapping, i.e., $M$ satisfies Definition~\ref{Def:DP}. Since the rewards of the arms are Bernoulli, we only need to consider $M(0)$ and $M(1)$. Let $f : \mathbb{R} \rightarrow \mathbb{R}^+$ be the PDF of $M(0)$ and $g : \mathbb{R} \rightarrow \mathbb{R}^+$ be the PDF of $M(1)$. Here, we allow points masses on $f$ and $g$, and if both $f$ and $g$ have point masses on $x$ with values $r$ and $s$, respectively, then we say $f(x) / g(x) = r/s$. 
		
		Now, we let $a$ be an arm with Bernoulli$(p)$ rewards and $b$ be an arm with Bernoulli$(q)$ rewards. Without loss of generality, we assume $p \geq q$. Let $h_a$ be the PDF of the output of CTB$(\epsilon)$ on arm $a$ and $h_b$ be the PDF of the output of CTB$(\epsilon)$ on $b$. We have
		\begin{align}
		\forall x\in \mathbb{R},\ h_a(x) = pf(x) + (1-p)g(x) \mbox{, and } h_b(x) = qf(x) + (1-q)g(x). \nonumber
		\end{align}
		Let $\Omega_f$ be the support of $f$ and $\Omega_g$ be the support of $g$. Since $M$ is $\epsilon$-DP, we have 
		\begin{align}
		\forall x \in \Omega_f \cup \Omega_g,\ |\log(f(x) / g(x))| \leq \epsilon, \nonumber
		\end{align}
		which also implies $\Omega_f = \Omega_g$. To simplify notation, we let $\Omega = \Omega_f = \Omega_g$. 
		
		The key to the proof is to show the following lemma.
		\begin{restatable}{lemma}{RestateKL}\label{Lm:Dkl}
			Let $0\leq p, q \leq 1$ and any two PDFs $f$ and $g$ with the same support $\Omega \subset \mathbb{R}$ and $\sup_{x\in\Omega}|\log(f(x)/g(x))| \leq \epsilon$ be given. Define $h_a := pf + (1-p)g$ and $h_b := qf + (1-q)g$. We have
			\begin{align}
				D_{\mbox{\tiny{KL}}}(h_a||h_b) \leq (e^\epsilon - e^{-\epsilon})^2(p-q)^2. \nonumber
			\end{align}
		\end{restatable}
	
		Thus, for any suboptimal arm $a$ and $\epsilon$-DP randomized mapping $M$, we have 
		\begin{align}
			D_{\mbox{\tiny{KL}}}(M(R^1_a)||M(R^1_{a^*})) \leq (e^\epsilon - e^{-\epsilon})^2(\mu_{a} - \mu_{a^*})^2 = (e^\epsilon - e^{-\epsilon})^2\Delta_a^2, \nonumber
		\end{align}
		where $a^*$ is the optimal arm.
		
		Since the bandit algorithm only has access to the private responses, i.e., $A^{t+1} \in \sigma(A^1,A^2,...,A^t,M(R^1),M(R^2),...,M(R^t))$ for any time $t$, by Theorem~2 in \cite{BinomialKLBound1989}, we conclude that if
		\begin{align}
			\sum_{a:\Delta_a > 0}\mathbb{E}N^T_a = o(T^\alpha) \mbox{ for every }\alpha > 0, \nonumber
		\end{align}
		then for any suboptimal arm $a$, 
		\begin{align}
			\liminf_{T\rightarrow \infty}\frac{\mathbb{E}N^T_a}{\log{T}}  \geq \frac{1}{D_{\mbox{\tiny{KL}}}(M(R_a^1)||M(R_{a^*}^1))}\geq \frac{1}{(e^\epsilon - e^{-\epsilon})^2\Delta_a^2}. \nonumber
		\end{align}
		If $\sum_{a:\Delta_a > 0}\mathbb{E}N^T_a = \Omega(T^\alpha)$ for some $\alpha > 0$, then for some arm $a$, $N^T_a$ is $\omega(\log{T})$, which implies that the regret $R(T)$ is $\omega(\log{T})$. Thus, we conclude that the regret $R(T)$ of any $\epsilon$-LDP policy must satisfy
		\begin{align}
			\liminf_{T\rightarrow\infty}\frac{R(T)}{\log{T}} = \liminf_{T\rightarrow\infty}\frac{\sum_{a:\Delta_a > 0}[\Delta_a\mathbb{E}N^T_a]}{\log{T}} \geq \frac{1}{(e^\epsilon - e^{-\epsilon})^2}\sum_{a:\Delta_a>0}\frac{1}{\Delta_a}. \nonumber
		\end{align}
		
		This completes the proof of Theorem~\ref{Thm:LB}.
	\end{proof}

\subsection{Proof of Theorem~\ref{Thm:LDP-UCB-L}}
\RestateLDPUCBL*
\begin{proof}[\textbf{Proof.}]
	The $\epsilon$-LDP of LDP-UCB-L follows from the the $\epsilon$-DP of CTL stated in Lemma~\ref{Lm:CTL}.
	
	\textbf{Distribution-dependent regret.} 
	Let $a^*$ be the arm with the largest mean reward and recall $\Delta_a = \mu_{a^*} - \mu_a$. An arm $a$ is said to be \textit{optimal} if $\Delta_a = 0$, and is said to be \textit{suboptimal} if $\Delta_{a} > 0$. For arm $a$ and time $t$, we use $R_a^t$ to denote the reward of the $t$-th pull and use $X_a^t$ denote the corresponding private response returned by CTL$(\epsilon)$. $R^t_a$ is with mean $\mu_a$ and support $[0,1]$, and $X^t_a = R^t_a + L^t_a$, where $(L^t_a, a\in[n], t\in\mathbb{Z}^+)$ are independent Laplace$(1/\epsilon)$ variables.
	
	By the Chernoff-Hoeffding inequality \cite{Hoeffding1963}, for any arm $a$, positive integer $k$, and time $t$, we have
	\begin{align}
	\mathbb{P}\Big\{\frac{1}{k}\sum_{r=1}^k{R^r_a} \geq \mu_{a} + \sqrt{\frac{2\log{t}}{k}}\Big\} \leq \exp\Big\{-2k\cdot\frac{2\log{t}}{k}\Big\} = t^{-4},\nonumber
	\end{align}
	and
	\begin{align}
	\mathbb{P}\Big\{\frac{1}{k}\sum_{r=1}^k{R^r_a} \leq \mu_{a} - \sqrt{\frac{2\log{t}}{k}}\Big\} \leq \exp\Big\{-2k\cdot\frac{2\log{t}}{k}\Big\} = t^{-4}.\nonumber
	\end{align}
	
	Also, for any time $t$ and positive integer $k > 4\log{t}$, setting $\nu = \sqrt{\sum_{r=1}^k(1/\epsilon)^2} = \sqrt{k}/\epsilon$ and $\lambda = \sqrt{({32k\log{t}})/{\epsilon^2}}$, we have $\lambda < \sqrt{8k^2/{\epsilon^2}} = 2\sqrt{2}\epsilon\nu^2 = 2\sqrt{2}k/\epsilon$, which by Lemma~\ref{Lm:LaplaceConcentration} implies
	\begin{align}
	\mathbb{P}\Big\{\sum_{r=1}^k{L^r_a} \geq k\sqrt{\frac{32\log{t}}{k\epsilon^2}}\Big\} \leq \exp\Big\{-\frac{\lambda^2}{8\nu^2}\Big\} =  \exp\Big\{-\frac{1}{8}\cdot\frac{\epsilon^2}{k}\cdot\frac{32k\log{t}}{\epsilon^2}\Big\} = t^{-4},\nonumber
	\end{align}
	and 
	\begin{align}
	\mathbb{P}\Big\{\sum_{r=1}^k{L^r_a} \leq -k\sqrt{\frac{32\log{t}}{k\epsilon^2}}\Big\} = \mathbb{P}\Big\{\sum_{r=1}^k{L^r_a} \geq k\sqrt{\frac{32\log{t}}{k\epsilon^2}}\Big\} \leq t^{-4},\nonumber
	\end{align}
	Therefore, for any arm $a$, time $t$, and positive integer $k > 4\log{t}$, we have
	\begin{align}
	& \mathbb{P}\Big\{u^t_a \geq \mu_a + 2\Big(\sqrt{\frac{2\log{t}}{N^t_a}} + \sqrt{\frac{32\log{t}}{N^t_a\epsilon^2}}\Big)\Big| N^t_a = k\Big\} \nonumber \\
	& = \mathbb{P}\Big\{\hat{\mu}^t_a \geq \mu_a + \sqrt{\frac{2\log{t}}{N^t_a}} + \sqrt{\frac{32\log{t}}{N^t_a\epsilon^2}}\Big| N^t_a = k\Big\} \nonumber \\
	& \leq \mathbb{P}\Big\{\frac{1}{k}\sum_{r=1}^k{R^t_a} \geq \mu_a + \sqrt{\frac{2\log{t}}{k}}\Big\} + \mathbb{P}\Big\{\sum_{r=1}^k{L^t_a} \geq k\sqrt{\frac{32\log{t}}{k\epsilon^2}} \Big\} \nonumber \\
	& \leq t^{-4} + t^{-4} = 2t^{-4}, \nonumber
	\end{align}
	and 
	\begin{align}
	\mathbb{P}\{u^t_a \leq \mu_a\mid N^t_a = k\} & = \mathbb{P}\Big\{\hat{\mu}^t_a \leq \mu_a - \sqrt{\frac{2\log{t}}{N^t_a}} - \sqrt{\frac{32\log{t}}{N^t_a\epsilon^2}}\Big| N^t_a = k\Big\} \nonumber \\
	& \leq \mathbb{P}\Big\{\frac{1}{k}\sum_{r=1}^k{R^t_a} \leq \mu_a - \sqrt{\frac{2\log{t}}{k}}\Big\} + \mathbb{P}\Big\{\sum_{r=1}^k{L^t_a} \leq -k\sqrt{\frac{32\log{t}}{k\epsilon^2}} \Big\} \nonumber \\
	& \leq t^{-4} + t^{-4} = 2t^{-4}. \nonumber
	\end{align}
	Here, we note that the above inequalities hold only if $k > 4\log{t}$ as required by Lemma~\ref{Lm:LaplaceConcentration}. This is the reason why we have Lines~5 and 6 in the algorithm LDP-UCB-L.
	
	Set
	\begin{align}
	v_{a,t} = \frac{8(1 + 4/\epsilon)^2\log{t}}{\Delta_a^2}, \nonumber
	\end{align}
	and we have 
	\begin{align}
	N^t_a > v_{a,t} \implies \mu_a + 2\Big(\sqrt{\frac{2\log{t}}{N^t_a}} + \sqrt{\frac{32\log{t}}{N^t_a\epsilon^2}}\Big) < \mu_a + \Delta_a = \mu_{a^*}.\nonumber
	\end{align}
	Also, since $\Delta_a \leq 1$, we have $v_{a,t} > 4\log{t}$. 
	
	For any suboptimal arm $a$ and time $t$, we have
	\begin{align}
	& \mathbb{P}\Big\{u^t_a \geq u^t_{a^*}\Big| N^t_a > v_{a,t}\Big\} \nonumber \\ 
	& \leq \sum_{r=\lfloor 1+ v_{a,t}\rfloor}^{t}\sum_{s=1}^t\Big[\mathbb{P}\Big\{u^t_a \geq \mu_{a^*}\Big|N^t_a = r\Big\} + \mathbb{P}\Big\{u^t_{a^*} \leq \mu_{a^*}\Big|N^t_{a^*} = s\Big\}\Big] \nonumber\\
	& \leq \sum_{r=\lfloor 1+ v_{a,t}\rfloor}^{t}\sum_{s=1}^t\Big[\mathbb{P}\Big\{u^t_a \geq \mu_a\!+\!2\Big(\sqrt{\frac{2\log{t}}{N^t_a}}\! + \!\sqrt{\frac{32\log{t}}{N^t_a\epsilon^2}}\Big)\Big|N^t_a = r\Big\}\!+\mathbb{P}\Big\{u^t_{a^*} \leq \mu_{a^*}\Big|N^t_{a^*} = s\Big\}\Big] \nonumber\\
	& \leq t\cdot t\cdot (t^{-4} + t^{-4}) = 4t^{-2}.\nonumber
	\end{align}
	
	Recall that $A^t$ is the $t$-th arm to be pulled and note that for any round $t > n$ and suboptimal arm $a$ with $N^t_a > 4\log{t}$, arm $a$ is pulled only if $u^t_a \geq u^t_{a^*}$. Therefore, we have 
	\begin{align}
	\mathbb{E}[N^T_a] =& \sum_{t=1}^T\Big[\mathbb{P}\{N^t_a\leq v_{a,t}\}\mathbb{P}\Big\{A^t = a \Big| N^t_a \leq v_{a,t}\Big\} + \mathbb{P}\{N^t_a > v_{a,t}\}\mathbb{P}\Big\{A^t = a \Big| N^t_a > v_{a,t}\Big\}\Big] \nonumber\\
	\leq & 1 + v_{a,T} + \sum_{t=1}^T\Big[1\cdot\mathbb{P}\Big\{A^t = a \Big| N^t_a > v_{a,t}\Big\}\Big] \nonumber\\
	\stackrel{(a)}{\leq} &  1 +  \frac{8(1 + 4/\epsilon)^2\log{T}}{\Delta_a^2} + \sum_{t=1}^T\mathbb{P}\Big\{u^t_a \geq u^t_{a^*}\Big| N^t_a > \frac{8(1 + 4/\epsilon)^2\log{t}}{\Delta_a^2}\Big\} \nonumber \\
	\leq & \frac{8(1 + 4/\epsilon)^2\log{T}}{\Delta_a^2} + 1 + 4\sum_{t=1}^Tt^{-2}\nonumber \\
	\leq & \frac{8(1 + 4/\epsilon)^2\log{T}}{\Delta_a^2} + 1 + \frac{2\pi^2}{3},\nonumber
	\end{align}
	where (a) is due to $v_{a,t} = \frac{8(1 + 4/\epsilon)^2\log{t}}{\Delta_a^2}$ and $v_{a,t} > 4\log{t}$ for all arms $a$ ($\Delta_a \leq 1$ always holds) and time $t$.
	
	Thus, the (expected) regret of LDP-UCB-L is at most
	\begin{align}
	\sum_{a:\Delta_a>0}\Big[\frac{8(1 + 4/\epsilon)^2\log{T}}{\Delta_a} + \Big(1 + \frac{2\pi^2}{3}\Big)\Delta_a\Big] \nonumber.
	\end{align}
	The proof of the distribution-dependent regret is complete.
	
	\textbf{Distribution-free regret.} Observe that
	\begin{align}
	\mathbb{E}[N_a(T)] \leq \frac{8(1 + 4/\epsilon)^2\log{T}}{\Delta_a^2} + 1 + \frac{2\pi^2}{3} =  O\Big(\frac{\log{T}}{\epsilon^2\Delta_a^2} + 1\Big). \nonumber
	\end{align} 
	Let $\alpha$ be a number in $(0,1)$. For all arms $a$ with $\Delta_a \leq \alpha$, the regret incurred by pulling these arms is upper bounded by $T\alpha$. For any arm $a$ with $\Delta_a > \alpha$, the expected regret incurred by pulling arm $a$ is upper bounded by $\mathbb{E}[N_a(T)\Delta_a] =O(\frac{1}{\epsilon^2\Delta_a}\log{T} + 1)$. Thus, the regret of LDP-UCB-L is at most
	\begin{align}
	T\alpha + O\Big(\frac{n}{\epsilon^2\alpha}\log{T} + n\Big).  \nonumber
	\end{align}
	By choosing $\alpha = \Theta(\sqrt{\frac{n\log{T}}{T\epsilon^2}})$ and recalling $T \geq n$, the regret is upper bounded by
	\begin{align}
	O\Big(T\sqrt{\frac{n\log{T}}{T\epsilon^2}} + \sqrt{\frac{T\epsilon^2}{n\log{T}}}\frac{n}{\epsilon^2}\log{T} + n\Big) = O\Big(\frac{1}{\epsilon}\sqrt{nT\log{T}}\Big). \nonumber
	\end{align}
	This completes the proof of the distribution-free regret, and the proof of Theorem~\ref{Thm:LDP-UCB-B} is complete.
\end{proof}

\subsection{Proof of Lemma~\ref{Lm:CTB}}
\RestateCTB*
\begin{proof}[\textbf{Proof.}]
	Let arm $a$ with mean reward $\mu_a$ and privacy parameter $\epsilon > 0$ be given. Let $R$ denote the reward received by pulling the arm and use $X$ to denote the output (returned value) of CTB$(\epsilon)$. We have $X = M_B(R)$. The value of $X$ is either $1$ or $0$, and thus, $X$ follows some Bernoulli distribution.
	
	Let $r, r'$ in $[0,1]$ be given. Observe that
	\begin{align}
	\mathbb{P}\{M_B(r) = 1\} = \frac{re^\epsilon + 1 - r}{e^\epsilon + 1} = \frac{1}{2} + (2r - 1) \cdot \frac{e^\epsilon - 1}{2(e^\epsilon + 1)}, \nonumber
	\end{align}
	while by $\mathbb{E}[R] = \mu_a$ the mean reward of arm $a$, implies 
	\begin{align}
	\mathbb{E}[X] = \mathbb{E}[M_B(R)] = \mathbb{E}\Big[\frac{1}{2} + (2R - 1) \cdot \frac{e^\epsilon - 1}{2(e^\epsilon + 1)}\Big] = \frac{1}{2} + (2\mu_a - 1) \cdot \frac{e^\epsilon - 1}{2(e^\epsilon + 1)}. \nonumber
	\end{align}
	This proves the mean of the returned value.
	
	Since $\mathbb{P}\{M(r) = 1\}$ is increasing on $r$, we have 
	\begin{align}
	\frac{\mathbb{P}\{M(r) = 1\}}{\mathbb{P}\{M(r') = 1\}} \leq \frac{\mathbb{P}\{M(1) = 1\}}{\mathbb{P}\{M(0) = 1\}} = \frac{e^\epsilon/(e^\epsilon + 1)}{1/(e^\epsilon + 1)} = e^\epsilon. \nonumber
	\end{align}
	Also, since $\mathbb{P}\{M(r) = 0\}$ is decreasing on $r$, we have
	\begin{align}
	\frac{\mathbb{P}\{M(r) = 0\}}{\mathbb{P}\{M(r') = 0\}} \leq \frac{\mathbb{P}\{M(0) = 0\}}{\mathbb{P}\{M(1) = 0\}} = \frac{e^\epsilon/(e^\epsilon + 1)}{1/(e^\epsilon + 1)} = e^\epsilon. \nonumber
	\end{align}
	Thus, by the definition of $\epsilon$-DP stated in Definition~\ref{Def:DP}, we conclude that CTB is $\epsilon$-DP. This completes the proof.
\end{proof}

\subsection{Proof of Theorem~\ref{Thm:LDP-UCB-B}}
\RestateLDPUCBB*
	\begin{proof}[\textbf{Proof.}]
	The $\epsilon$-LDP of LDP-UCB-B follows from the $\epsilon$-DP of CTB stated in Lemma~\ref{Lm:CTB}.
	
	\textbf{Distribution-dependent regret.} According to Lemma~\ref{Lm:CTB}, for any arm $a$, the private response generated by CTB$(a,\epsilon)$ follows the Bernoulli$(\mu_{a,\epsilon})$ distribution, where 
	\begin{align}
	\mu_{a,\epsilon} := \frac{1}{2} + \frac{(2\mu_a-1)(e^\epsilon-1)}{2(e^\epsilon + 1)}. \nonumber
	\end{align} 
	Define $\mu^*_\epsilon := \max_{a\in[n]}\mu_{a.\epsilon}$ and 
	\begin{align}
	\Delta_{a,\epsilon} := \mu^*_\epsilon - \mu_{a,\epsilon} = \frac{e^\epsilon - 1}{e^\epsilon + 1} \cdot \Delta_a \nonumber 
	\end{align}
	for any arm $a$ in $[n]$. We note that $\Delta_a > 0$ if and only if $\Delta_{a,\epsilon} > 0$. An arm $a$ is said to be \textit{optimal} if $\Delta_a = 0$, and is said to be \textit{suboptimal} if $\Delta_{a} > 0$. 
	
	For any suboptimal arm $a$ and time $t$, we have
	\begin{align}
	N^t_a > \frac{8\log{T}}{\Delta_{a,\epsilon}^{2}} \implies \mu_{a,\epsilon} + 2\sqrt{\frac{2\log{t}}{N^t_a}} < \mu_{a,\epsilon} + \Delta_{a,\epsilon} = \mu^*_\epsilon. \nonumber
	\end{align}
	Also, by the Chernoff-Hoeffding Inequality \cite{Hoeffding1963}, for any suboptimal arm $a$, time $t$, and positive integer $k \leq t$, we have
	\begin{align}
	\mathbb{P}\Big\{u^t_a \geq \mu_{a,\epsilon} + 2\sqrt{\frac{2\log{t}}{N^t_a}}\Big| N^t_a = k\Big\} 
	&= \mathbb{P}\Big\{\hat{\mu}^t_a - \sqrt{\frac{2\log{t}}{N^t_a}} \geq \mu_{a,\epsilon}\Big| N^t_a = k\Big\} \nonumber \\
	&\leq \exp\Big\{-2k\cdot\frac{2\log{t}}{k}\Big\} = t^{-4},\nonumber
	\end{align}
	and
	\begin{align}
	\mathbb{P}\Big\{u^t_a \leq \mu_{a,\epsilon}\Big| N^t_a = k\Big\} = \mathbb{P}\Big\{\hat{\mu}^t_a + \sqrt{\frac{2\log{t}}{N^t_a}} \leq \mu_{a,\epsilon}\Big| N^t_a = k\Big\} \leq \exp\Big\{-2k\cdot\frac{2\log{t}}{k}\Big\} = t^{-4}. \nonumber
	\end{align}
	Thus, we have
	\begin{align}
	& \mathbb{P}\Big\{u^t_a \geq u^t_{a^*}\Big| N^t_a > \frac{8\log{T}}{\Delta_{a,\epsilon}^{2}}\Big\} \nonumber \\ 
	& \leq \sum_{r=\lfloor 1+ \frac{8\log{T}}{\Delta_{a,\epsilon}^{2}}\rfloor}^{t}\sum_{s=1}^t\Big[\mathbb{P}\Big\{u^t_a \geq \mu_{a^*,\epsilon}\Big|N^t_a = r\Big\} + \mathbb{P}\Big\{u^t_{a^*} \leq \mu_{a^*,\epsilon}\Big|N^t_{a^*} = s\Big\}\Big] \nonumber\\
	& \leq \sum_{r=\lfloor 1+ \frac{8\log{T}}{\Delta_{a,\epsilon}^{2}}\rfloor}^{t}\sum_{s=1}^t\Big[\mathbb{P}\Big\{u^t_a \geq \mu_{a,\epsilon} + 2\sqrt{\frac{2\log{t}}{N^t_a}}\Big|N^t_a = r\Big\} + \mathbb{P}\Big\{u^t_{a^*} \leq \mu_{a^*,\epsilon}\Big|N^t_{a^*} = s\Big\}\Big] \nonumber\\
	& \leq t\cdot t\cdot (t^{-4} + t^{-4}) = 2t^{-2}.\nonumber
	\end{align}
	Note that at any round $t > n$ and for any suboptimal arm $a$, $a$ is pulled only if $u^t_a \geq u^t_{a^*}$. Thus, setting $v_{a,t} = \frac{8\log{t}}{\Delta_{a,\epsilon}^2}$, we have 
	\begin{align}
	\mathbb{E}[N^T_a] =& \sum_{t=1}^T\Big[\mathbb{P}\{N^t_a\leq v_{a,t}\}\mathbb{P}\Big\{A^t = a \Big| N^t_a \leq v_{a,t}\Big\} + \mathbb{P}\{N^t_a > v_{a,t}\}\mathbb{P}\Big\{A^t = a \Big| N^t_a > v_{a,t}\Big\}\Big] \nonumber\\
	\leq & 1 + v_{a,T} + \sum_{t=1}^T\Big[1\cdot\mathbb{P}\Big\{A^t = a \Big| N^t_a > v_{a,t}\Big\}\Big] \nonumber\\
	\leq & 1 + \frac{8\log{T}}{\Delta_{a,\epsilon}^{2}} + \sum_{t=1}^T\mathbb{P}\Big\{u^t_a \geq u^t_{a^*}\Big| N^t_a > \frac{8\log{T}}{\Delta_{a,\epsilon}^{2}}\Big\} \nonumber \\
	\leq & \frac{8\log{T}}{\Delta_{a,\epsilon}^{2}} + 1 + 2\sum_{t=1}^Tt^{-2}\nonumber \leq \frac{8\log{T}}{\Delta_{a,\epsilon}^2} + 1 + \frac{\pi^2}{3}.\nonumber
	\end{align}
	%
	%
	Thus, the (expected) regret of LDP-UCB is at most
	\begin{align}
	\sum_{a:\Delta_{a,\epsilon}>0}\Big[\frac{8}{\Delta_{a,\epsilon}^2}\log{T} + 1 + \frac{\pi^2}{3}\Big]\Delta_a = \sum_{a:\Delta_a>0}\Big[\frac{8}{\Delta_a}\Big(\frac{e^\epsilon+1}{e^\epsilon-1}\Big)^2\log{T} + \Big(1 + \frac{\pi^2}{3}\Big)\Delta_a\Big] \nonumber.
	\end{align}
	The proof of the distribution-dependent regret is complete.
	
	\textbf{Distribution-free regret.} Similar to the proof of the distribution-dependent regret, we have $\Delta_{a,\epsilon} = (\frac{e^\epsilon-1}{e^\epsilon+1})\Delta_a = \Omega(\epsilon\Delta_a)$ and 
	\begin{align}
	\mathbb{E}[N_a(T)] \leq \frac{8}{\Delta_{a,\epsilon}^2}\log{T} + 1 + \frac{\pi^2}{3} =  \frac{8}{\Delta_a^2}\Big(\frac{e^\epsilon+1}{e^\epsilon-1}\Big)^2\log{T} + O(1) =  O\Big(\frac{\log{T}}{\epsilon^2\Delta_a^2} + 1\Big). \nonumber
	\end{align} 
	Let $\alpha$ be a number in $(0,1)$. For all arms $a$ with $\Delta_a \leq \alpha$, the regret incurred by pulling these arms is upper bounded by $T\alpha$. For any arm $a$ with $\Delta_a > \alpha$, the expected regret incurred by pulling arm $a$ is upper bounded by $\mathbb{E}[N_a(T)\Delta_a] = O(\frac{1}{\epsilon^2\Delta_a}\log{T} + 1)$. Thus, the regret of LDP-UCB-B is at most
	\begin{align}
	T\alpha + O\Big(\frac{n}{\epsilon^2\alpha}\log{T} + n\Big).  \nonumber
	\end{align}
	By choosing $\alpha = \Theta(\sqrt{\frac{n\log{T}}{T\epsilon^2}})$ and recalling $T \geq n$, the regret is upper bounded by
	\begin{align}
	O\Big(T\sqrt{\frac{n\log{T}}{T\epsilon^2}} + \sqrt{\frac{T\epsilon^2}{n\log{T}}}\frac{n}{\epsilon^2}\log{T} + n\Big) = O\Big(\frac{1}{\epsilon}\sqrt{nT\log{T}}\Big). \nonumber
	\end{align}
	This completes the proof of the distribution-free regret, and the proof of Theorem~\ref{Thm:LDP-UCB-B} is complete.
\end{proof}

\subsection{Proof of Lemma~\ref{Lm:Sigmoid}}
\RestateSigmoid*
\begin{proof}[\textbf{Proof.}]
	Let $F$ be the cumulative probability function of $\mathcal{N}$ (here, we allow point mass on $\mathcal{N}$). Note that $\lambda, \mu \in [0,1]$, and we have
	\begin{align}
	\mathbb{E}[s(X) - s(Y)] =& \int_{-\infty}^{+\infty}\frac{1}{1+e^{-x-\mu}}-\frac{1}{1+e^{-x-\lambda}}\ \mathrm{d}F(x) \nonumber\\
	= & \int_{-\infty}^{+\infty}\frac{e^{-x}(e^{-\mu}-e^{-\lambda})}{(1+e^{-x-\lambda})(1+e^{-x-\mu})}\ \mathrm{d}F(x) \nonumber\\
	\geq & \int_{-\infty}^{+\infty}\frac{e^{-x}(e^{-\mu}-e^{-\lambda})}{(1+e^{-x})^2}\ \mathrm{d}F(x). \nonumber
	\end{align}
	Since the function $e^{-x}$ is decreasing and convex on $x$ and $(e^{-x})'|_{x=1} = -e^{-1}$, we have
	\begin{align}
	e^{-\mu} - e^{-\lambda} \geq e^{-(1+\mu-\lambda)} - e^{-1} \geq (e^{-x})'|_{x=1}\cdot (\mu-\lambda) = e^{-1}(\lambda - \mu). \nonumber
	\end{align}
	Let $Z$ be a random variable that follows $\mathcal{N}$, i.e., the CDF of $Z$ is $F$. For $a \in (0,1/4)$, solving the equation and we get
	\begin{align}
	\frac{e^{-x}}{(1+e^{-x})^2} = a \iff x = \pm \log\Big(\frac{1}{2a} - 1 + \sqrt{\frac{1}{4a^2}-\frac{1}{a}}\Big).\nonumber
	\end{align}
	Here, we let $v_a$ denote the positive part of the right-hand side. Since $\frac{e^{-x}}{(1+e^{-x})^2}$ decreases with $x$, $\frac{e^{-x}}{(1+e^{-x})^2} \geq a$ if and only if $|x| \leq v_a$. Therefore, we have
	\begin{align}
	\mathbb{P}\Big\{\frac{e^{-Z}}{(1+e^{-Z})^2}\geq a\Big\} =& \mathbb{P}\{|Z| \leq v_a\} \stackrel{(a)}{\geq} 1 - 2e^{-{v_a^2}/{2}},\nonumber 
	\end{align}
	where (a) is due to the property of sub-Gaussian distributions.
	Thus,
	\begin{align}
	\mathbb{E}\Big[\frac{e^{-Z}}{(1+e^{-Z})^2}\Big] \geq \sup_{a\in[0,1/4]}a(1 - 2e^{-v_a^2/2}) \stackrel{a = 0.107}{\geq} 0.0765. \nonumber 
	\end{align}
	This implies 
	\begin{align}
	\int_{-\infty}^{+\infty}\frac{e^{-x}}{(1+e^{-x})^2}\ \mathrm{d}F(x) \geq 0.0765. \nonumber
	\end{align}
	Along with $e^{-\mu} - e^{-\lambda} \geq e^{-1}(\lambda - \mu)$, we get
	\begin{align}
	\mathbb{E}[s(X) - s(Y)] \geq 0.0765e^{-1}(\lambda - \mu). \nonumber
	\end{align}
	Specifically, if the noises are Gaussian with mean $0$ and variance $1$, then by numerically computing the integration, we have 
	\begin{align}
	\mathbb{E}[s(X) - s(Y)] \geq 0.2066e^{-1}(\lambda - \mu). \nonumber
	\end{align}
	The proof of Lemma~\ref{Lm:Sigmoid} is complete.
\end{proof}

\subsection{Proof of Lemma~\ref{Lm:CTL-S}}
\RestateCTLS*
\begin{proof}[\textbf{Proof.}]
	The $\epsilon$-DP of CTL-S follows from the $\epsilon$-DP of CTL stated in Lemma~\ref{Lm:CTL}. The difference between the expected responses of CTL-S$(\epsilon)$ on arms $a$ and $b$ follows from Lemma~\ref{Lm:Sigmoid}. This completes the proof of Lemma~\ref{Lm:CTL-S}.
\end{proof}

\subsection{Proof of Lemma~\ref{Lm:CTB-S}}
\RestateCTBS*
\begin{proof}[\textbf{Proof.}]
	The $\epsilon$-DP of CTB-S follows from the $\epsilon$-DP of CTB stated in Lemma~\ref{Lm:CTB}.
	
	Let arms $a$ and $b$ with $\mu_a \geq \mu_b$ be given. Let $R_a$ be the reward of some pull of arm $a$ and $R_b$ be that of arm $b$. Let $S_a = s(R_a)$ and $S_b = s(R_b)$, where $s(\cdot)$ is the Sigmoid function. Let $X_a$ be the private response of CTB-S$(\epsilon)$ on arm $a$ and $X_b$ be that on arm $b$, i.e., $X_a = M_B(S_a)$ and $X_b = M_B(S_b)$. By Lemma~\ref{Lm:Sigmoid}, we have $\mathbb{E}[S_a-S_b] \geq c_s(\mu_a - \mu_b)$. Also, since $S_a$ and $S_b$ are bounded in $[0,1]$, by the property of $M_B(\cdot)$ stated in Lemma~\ref{Lm:CTB}, we have $\mathbb{E}[X_a - X_b] \geq c_s(\frac{e^\epsilon-1}{e^\epsilon+1})(\mu_a - \mu_b) = c_s(\mu_{a,\epsilon} - \mu_{b,\epsilon})$. This completes the proof of Lemma~\ref{Lm:CTB-S}.
\end{proof}

\subsection{Proof of Lemma~\ref{Lm:Dkl}}
\RestateKL*
\begin{proof}[\textbf{Proof.}]
	To prove this lemma, we introduce a fact about $D_{\mbox{\tiny{KL}}}$. 
	\begin{fact}[Lemma~1 in \cite{KLInequality2001}]\label{Fact:KLBound}
		For two PDFs $f_1$ and $g_1$ with the same support $\Omega$, if for all $x$ in $\Omega$ we have $0 < r \leq f_1(x) / g_1(x) \leq R < \infty$, then the following holds
		\begin{align}
		0 \leq D_{\mbox{\tiny{KL}}}(f_1||g_1) \leq \frac{(R-r)^2}{4rR}, \nonumber 
		\end{align}
		where the two equalities hold if and only if $f_1$ and $g_1$ are equal almost surely.
	\end{fact}
	
	Define $r(x) = f_1(x) / g_1(x)$, and we have $e^{-\epsilon} \leq r(x) \leq e^\epsilon$. For any $x$ in $\Omega$, we have
	\begin{align}
	\frac{h_a(x)}{h_b(x)} = \frac{pf(x) + (1-p)g(x)}{qf(x) + (1-q)g(x)} = \frac{pr(x) + 1-p}{qr(x) + 1-q}.\nonumber
	\end{align}
	Define 
	\begin{align}
	R := \frac{pe^\epsilon + 1-p}{qe^\epsilon + 1-q}\mbox{ and } r := \frac{pe^{-\epsilon} + 1-p}{qe^{-\epsilon} + 1-q}.\nonumber
	\end{align}
	If $p \geq q$, then $h_a(x) / h_b(x)$ is non-decreasing with $r(x)$, which implies
	\begin{align}
	r \leq h_a(x) / h_b(x) \leq R.\nonumber
	\end{align}
	If $p < q$, then $h_a(x) / h_b(x)$ is non-increasing with $r(x)$, which implies
	\begin{align}
	R \leq h_a(x) / h_b(x) \leq r.\nonumber
	\end{align}
	In both cases, by Fact~\ref{Fact:KLBound}, we have $D_{\mbox{\tiny{KL}}}(h_a||h_b) \leq \frac{(R-r)^2}{4rR}$, where
	\begin{align}
	R - r = & \frac{pe^{\epsilon} + 1-p}{qe^{\epsilon} + 1-q} - \frac{pe^{-\epsilon} + 1-p}{qe^{-\epsilon} + 1-q} \nonumber \\
	= & \frac{(pe^{\epsilon} + 1-p)(qe^{-\epsilon} + 1-q) - (qe^{\epsilon} + 1-q)(pe^{-\epsilon} + 1-p)}{(qe^{\epsilon} + 1-q)(qe^{-\epsilon} + 1-q)} \nonumber \\
	= & \frac{(p-q)(e^\epsilon - e^{-\epsilon})}{(qe^{\epsilon} + 1-q)(qe^{-\epsilon} + 1-q)},\nonumber
	\end{align}
	and
	\begin{align}
	rR = & \frac{pe^\epsilon + 1-p}{qe^\epsilon + 1-q} \cdot \frac{pe^{-\epsilon} + 1-p}{qe^{-\epsilon} + 1-q} \nonumber \\
	= & \frac{p^2 + p(1-p)(e^\epsilon + e^{-\epsilon}) + (1-p)^2}{(qe^{\epsilon} + 1-q)(qe^{-\epsilon} + 1-q)}. \nonumber
	\end{align}
	Therefore, we have
	\begin{align}
	D_{\mbox{\tiny{KL}}}(h_a||h_b) \leq & \frac{1}{4}\cdot \Big(\frac{(p-q)(e^\epsilon - e^{-\epsilon})}{(qe^{\epsilon} + 1-q)(qe^{-\epsilon} + 1-q)}\Big)^2 \cdot \frac{(qe^{\epsilon} + 1-q)(qe^{-\epsilon} + 1-q)}{p^2 + p(1-p)(e^\epsilon + e^{-\epsilon}) + (1-p)^2} \nonumber \\
	= & \frac{1}{4}\cdot\frac{(p-q)^2(e^\epsilon - e^{-\epsilon})^2}{(qe^{\epsilon} + 1-q)(qe^{-\epsilon} + 1-q)}\cdot\frac{1}{p^2 + p(1-p)(e^\epsilon + e^{-\epsilon}) + (1-p)^2}\nonumber \\
	= & \frac{1}{4}\cdot\frac{(p-q)^2(e^\epsilon - e^{-\epsilon})^2}{q^2 + q(1-q)(e^\epsilon+e^{-\epsilon})+(1-q)^2}\cdot\frac{1}{p^2 + p(1-p)(e^\epsilon + e^{-\epsilon}) + (1-p)^2}\nonumber \\
	\stackrel{(a)}{\leq} & \frac{1}{4}\cdot\frac{(p-q)^2(e^\epsilon - e^{-\epsilon})^2}{1/2 + q(1-q)(e^\epsilon+e^{-\epsilon})}\cdot\frac{1}{1/2 + p(1-p)(e^\epsilon + e^{-\epsilon})} \nonumber\\
	\leq & \frac{1}{4} \cdot \frac{(p-q)^2(e^\epsilon - e^{-\epsilon})^2}{1/2} \cdot \frac{1}{1/2} \nonumber \\
	= & (p-q)^2(e^\epsilon - e^{-\epsilon})^2, \nonumber
	\end{align}
	where (a) is due to $x^2 + (1-x)^2 \geq 1/2$ for any $x$. This completes the proof of Lemma~\ref{Lm:Dkl}.
\end{proof}

	\begin{figure}[h]
	\begin{subfigure}[b]{0.32\textwidth}
		\includegraphics[scale=0.25]{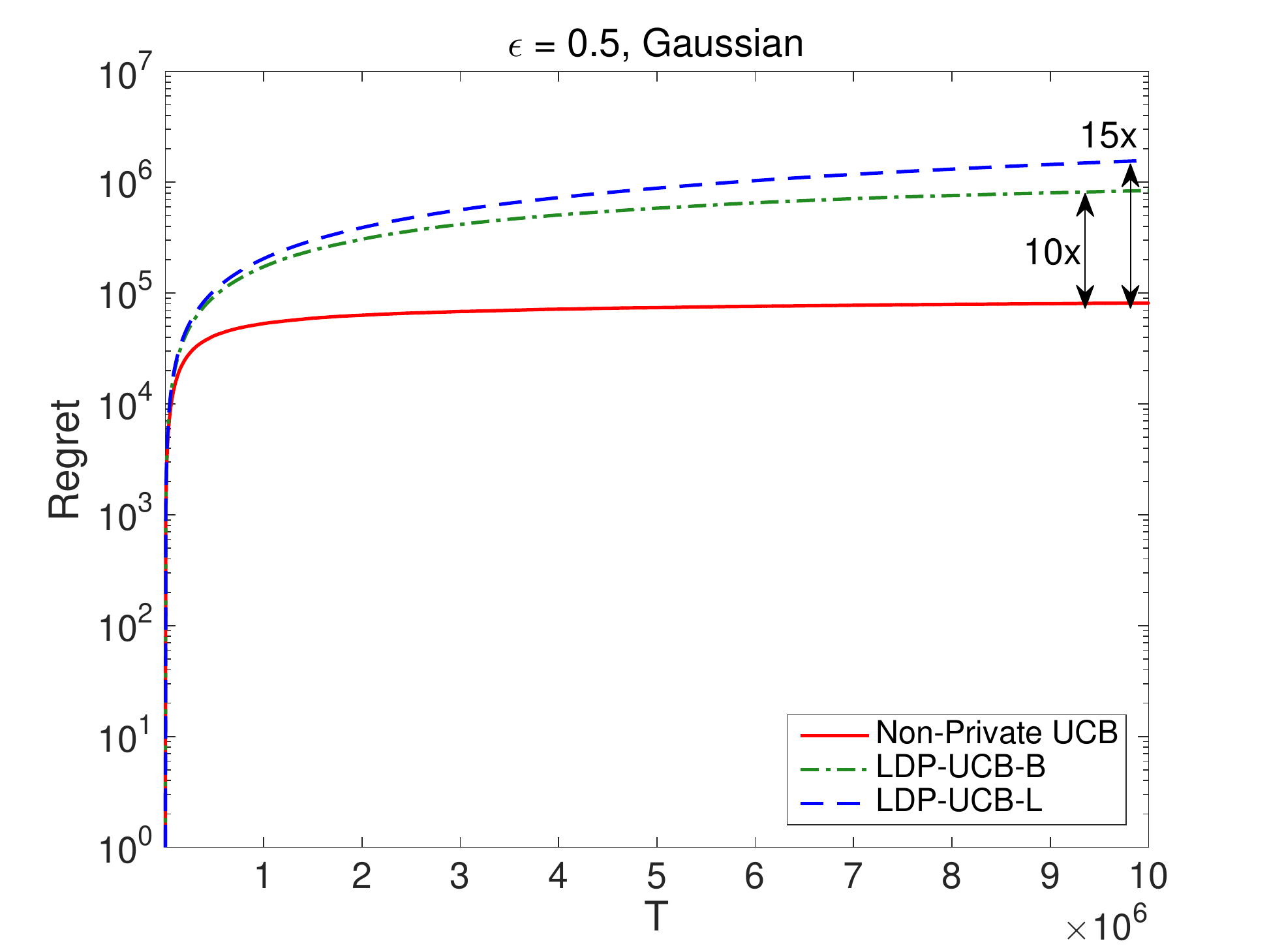}
		\caption{$\epsilon = 0.5$.}
	\end{subfigure}\ \ 
	\begin{subfigure}[b]{0.32\textwidth}
		\includegraphics[scale=0.25]{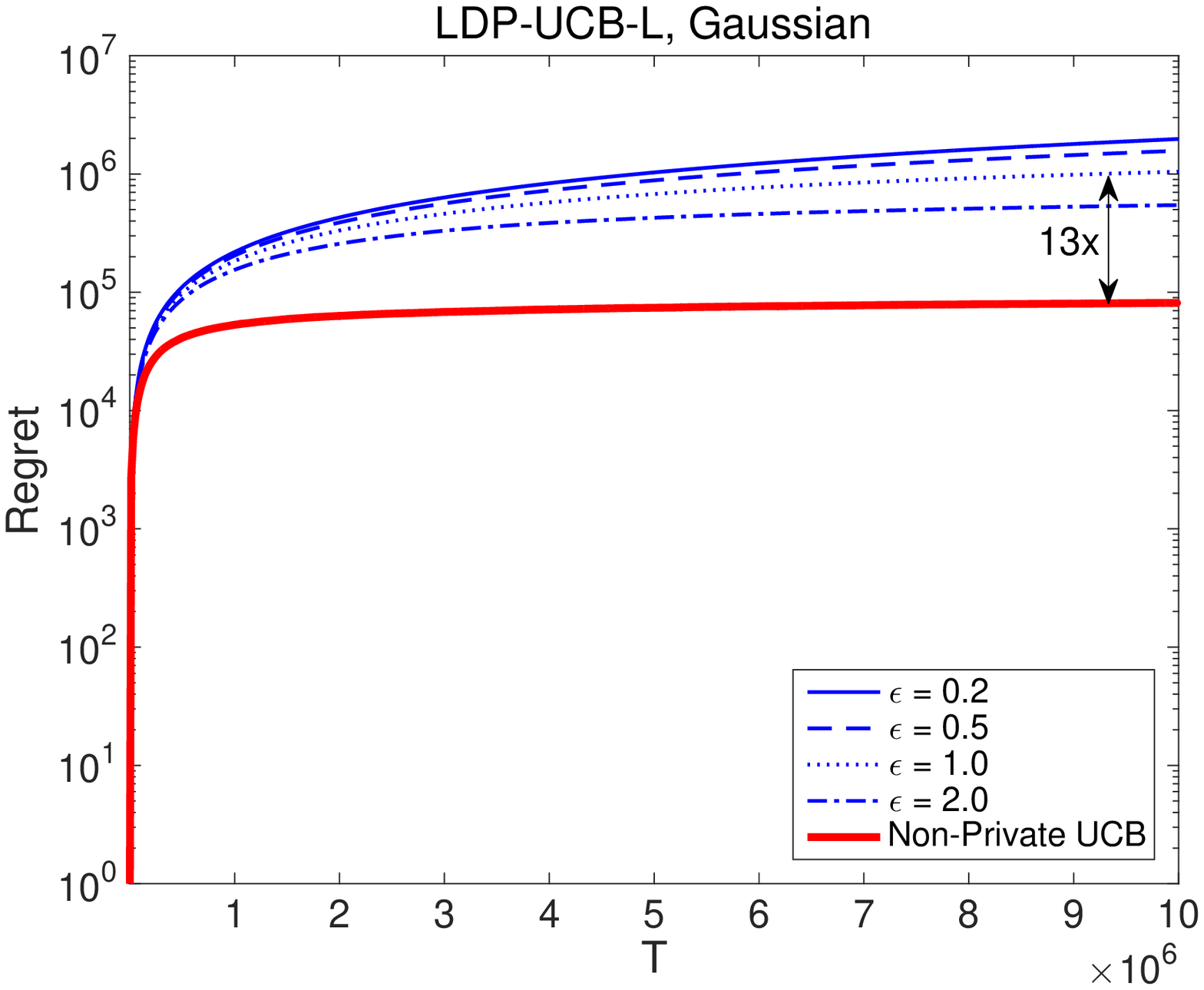}
		\caption{LDP-UCB-LS, vary $\epsilon$.}
	\end{subfigure}\ \ 
	\begin{subfigure}[b]{0.32\textwidth}
		\includegraphics[scale=0.25]{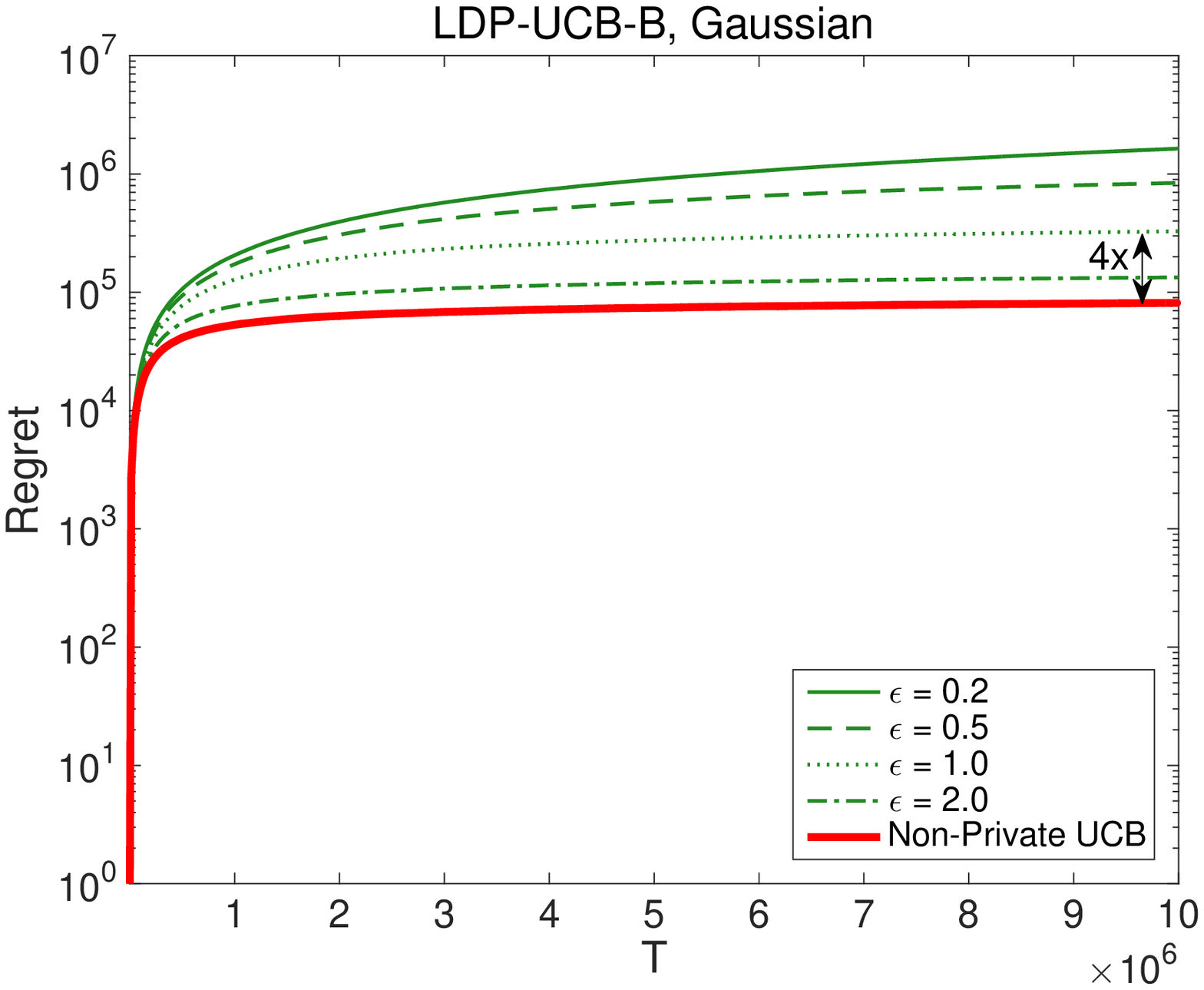}
		\caption{LDP-UCB-BS, vary $\epsilon$.}
	\end{subfigure}
	\caption{Numerical results for LDP-UCB-LS and LDP-UCB-BS. Gaussian arm for all subfigures.}\label{fig:unbounded}
\end{figure}

\section{Additional numerical results}

	In this section, we present the empirical performance of LDP-UCB-LS and LDP-UCB-BS, and the results are illustrated in Figure~\ref{fig:unbounded}. In the experiments, there are $n = 20$ arms. The best arm has mean reward $0.9$; five arms have mean rewards $0.8$; five arms have mean rewards $0.7$; five arms have mean rewards $0.6$; and four arms have mean rewards $0.5$. The rewards of all arms follow the Gaussian distributions with variance one. We also include the non-private UCB algorithm UCB1 in \cite{FiniteAnalysis2002} as a baseline. For fair comparisons, we also add the Sigmoid preprocessing on the rewards when running the non-private UCB algorithm.

	In Figure~\ref{fig:unbounded} (a), we set $\epsilon = 0.5$ and compare different algorithms. In (b), we vary the values of $\epsilon$ to compare the performance of LDP-UCB-LS under different values of $\epsilon$. In Figure~\ref{fig:unbounded} (b), we vary the value of $\epsilon$ to evaluate the performance of LDP-UCB-LS under different values of  $\epsilon$. In Figure~\ref{fig:unbounded} (c), we vary the value of $\epsilon$ to evaluate the performance of LDP-UCB-BS. From the results, we can see that the ratios of the regrets of LDP-UCB-LS and LDP-UCB-BS to that of the non-private UCB is similar to that of the bounded arms. In theory, when $\epsilon = 0.5$, the upper bounds of the ratios are $(1+\frac{4}{\epsilon})^2 = 81$ for LDP-UCB-LS and $(\frac{e^\epsilon+1}{e^\epsilon - 1})^2 = 17$ for LDP-UCB-BS, larger than the empirical results $15$ and $10$ shown in Figure~\ref{fig:unbounded} (a), which confirms our theoretical results.

\end{appendices}

\end{document}